\documentclass[runningheads,a4paper]{llncs}
\usepackage{amssymb,amsmath}
\usepackage{graphicx}
\usepackage{pifont}

\usepackage[polish,english]{babel}
\usepackage[T1,nomathsymbols]{polski}
\usepackage[utf8]{inputenc}
\usepackage[T1]{fontenc}

\usepackage{tikz}
\usetikzlibrary{arrows,backgrounds,automata,topaths,patterns,decorations.pathreplacing,decorations.pathmorphing}
\usetikzlibrary{calc}
\usetikzlibrary{fit}
\usetikzlibrary{arrows}

\usepackage[sort,numbers,sectionbib]{natbib} 

\newcommand{\notapp}{\emph{n/a}}
\newcommand{\fail}{\emph{fail}}
\newcommand{\timeout}{\emph{time out}}

\usepackage{url}
\urldef{\mailsa}\path|{alfred.hofmann, ursula.barth, ingrid.haas, frank.holzwarth,|
\urldef{\mailsb}\path|anna.kramer, leonie.kunz, christine.reiss, nicole.sator,|
\urldef{\mailsc}\path|erika.siebert-cole, peter.strasser, lncs}@springer.com|    
\newcommand{\keywords}[1]{\par\addvspace\baselineskip
\noindent\keywordname\enspace\ignorespaces#1}

\usepackage[sf,SF]{subfigure}

\newtheorem{mytheorems}{Theorem}
\newtheorem{myconditions}{Condition}

\newcommand{\VectorC}[2]{\left( \begin{array}{c} #1 \\ #2 \end{array} \right)}

\newcommand{\casei}{\operatorname{(i)}}
\newcommand{\caseii}{\operatorname{(ii)}}
\newcommand{\caseiii}{\operatorname{(iii)}}
\newcommand{\caseiv}{\operatorname{(iv)}}

\usepackage{xcolor}

\definecolor{DarkBlue}{RGB}{7,72,110}
\definecolor{LightBlue}{RGB}{205,237,249}
\definecolor{LightRed}{RGB}{255,102,102}
\definecolor{DarkRed}{RGB}{153,0,0}
\definecolor{LightGreen}{RGB}{178,205,102}
\definecolor{DarkGreen}{RGB}{0,51,0}

\newcommand{\Pred}[1]{\operatorname{#1}}

\newcommand{\Rel}{\operatorname{Rel}}

\newcommand{\Sym}{\operatorname{Sym}}
\newcommand{\Reals}{\operatorname{\mathbb{R}}}
\newcommand{\Let}{\operatorname{let}}

\usepackage{hyperref}

\hypersetup{
    bookmarks=true,         
    pdftoolbar=true,        
    pdfmenubar=true,        
    pdffitwindow=false,     
    pdfstartview={FitH},    
    pdftitle={Spatial Symmetry Driven Pruning},    
    pdfauthor={Todo},     
    pdfnewwindow=true,      
    colorlinks=true,       
    linkcolor=blue,        
    linkbordercolor=white,
    citecolor=blue!50!black,        
    urlcolor=blue!80!black           
}

\newcommand{\bul}{$\triangleright~$}

\begin{document}

\mainmatter  

\title{Spatial Symmetry Driven Pruning Strategies for\\Efficient Declarative Spatial Reasoning}

\titlerunning{Spatial Symmetry Driven Pruning}

%
%
\author{ Carl Schultz\inst{1,3} \and Mehul Bhatt\inst{2,3} 
}
\authorrunning{}


\institute{Institute for Geoinformatics, University of M\"{u}nster, Germany
\and
Department of Computer Science, University of Bremen, Germany
\and
The DesignSpace Group
}


%
%

\toctitle{Spatial Symmetry Based Pruning}
\tocauthor{}
\maketitle

\begin{abstract}
Declarative spatial reasoning denotes the ability to (declaratively) specify and solve real-world problems related to geometric and qualitative spatial representation and reasoning within standard knowledge representation and reasoning (KR) based methods (e.g., logic programming and derivatives). One approach for encoding the semantics of spatial relations within a declarative programming framework is by systems of polynomial constraints. However, solving such constraints is computationally intractable in general (i.e. the theory of real-closed fields). 

\smallskip

We present a new algorithm, implemented within the declarative spatial reasoning system CLP(QS), that drastically improves the performance of deciding the consistency of spatial constraint graphs over conventional polynomial encodings. We develop pruning strategies founded on spatial symmetries that form equivalence classes (based on affine transformations) at the qualitative spatial level. Moreover, pruning strategies are themselves formalised as knowledge about the properties of space and spatial symmetries. We evaluate our algorithm using a range of benchmarks in the class of contact problems, and proofs in mereology and geometry. The empirical results show that CLP(QS) with knowledge-based spatial pruning outperforms conventional polynomial encodings by orders of magnitude, and can thus be applied to problems that are otherwise unsolvable in practice.

\end{abstract}

 
\section{Introduction}
Knowledge representation and reasoning (KR) about \emph{space} may be formally interpreted within diverse frameworks such as: (a) analytically founded geometric reasoning \& constructive (solid) geometry \citep{Kapur:1989:GR:107262,pesant1999reasoning,owen1991algebraic}; (b) relational algebraic semantics of `qualitative spatial calculi' \citep{ligozat-book}; and (c) by axiomatically constructed formal systems of mereotopology and mereogeometry \citep{hdbook-spatial-logics}. Independent of formal semantics, commonsense spatio-linguistic abstractions offer a human-centred and cognitively adequate mechanism for logic-based automated reasoning about spatio-temporal information \citep{Bhatt-Schultz-Freksa:2013}. 


\noindent\bul\textbf{Declarative Spatial Reasoning}\quad In the recent years, \emph{declarative spatial reasoning} has been developed as a high-level commonsense spatial reasoning paradigm aimed at (declaratively) specifying and solving real-world problems related to geometric and qualitative spatial representation and reasoning \cite{bhatt-et-al-2011}. A particular manifestation of this paradigm is the constraint logic programming based  CLP(QS) spatial reasoning system \cite{bhatt-et-al-2011,schultz-bhatt-2012,DBLP:conf/ecai/SchultzB14} (Section \ref{sec:dsr}).

\medskip

\noindent\bul\textbf{Relational Algebraic Qualitative Spatial Reasoning}\quad The state of the art in qualitative spatial reasoning using relational algebraic methods \citep{ligozat-book} has resulted in prototypical algorithms and black-box systems that do not integrate with KR languages, such as those dealing with semantics and conceptual knowledge necessary for handling {background knowledge}, action \& change, relational learning, rule-based systems etc. Furthermore, relation algebraic qualitative spatial reasoning (e.g. LR \cite{ligozat1993qualitative}), while efficient, is incomplete in general \cite{ligozat-book,ladkin1994binary,leecomplexity-ecai2014}.\footnote{\emph{Incompleteness} refers to the inability of a spatial reasoning method to determine whether a given network of qualitative spatial constraints is consistent or inconsistent in general. Relation-algebraic spatial reasoning (i.e. using algebraic closure based on weak composition) has been shown to be incomplete for a number of spatial languages and cannot guarantee \emph{consistency} in general, e.g. relative directions \cite{leecomplexity-ecai2014} and containment relations between linearly ordered intervals \cite{ladkin1994binary}, Theorem 5.9.} Alternatively, constraint logic programming based systems such as CLP(QS)  \cite{bhatt-et-al-2011} and others (see \cite{pesant1994quad,bouhineau1996solving,pesant1999reasoning,bouhineau1999application,DBLP:conf/pods/KanellakisKR90,Grumbach97}) adopt an analytic geometry approach where spatial relations are encoded as systems of polynomial constraints;\footnote{We emphasise that this analytic geometry approach that we also adopt is not \emph{qualitative spatial reasoning} in the relation algebraic sense; the foundations are similar (i.e. employing a finite language of spatial relations that are interpreted as infinite sets of configurations, determining consistency in the complete absence of numeric information, and so on) but the methods for determining consistency etc. come from different branches of spatial reasoning.} while these methods are sound and complete (see Section\;\ref{sec:foundations}), they have prohibitive computational complexity, $O(c_1^{c_2^n})$ in the number of polynomial variables $n$, meaning that even relatively simple problems are not solved within a practical amount of time via ``naive'' or direct encodings as polynomial constraints, i.e. encodings that lack common-sense knowledge about spatial objects and relations. On the other hand, highly efficient and specialised geometric theorem provers (e.g. \cite{chou1988mechanical}) and geometric constraint solvers (e.g. \cite{owen1991algebraic,hadas2000role}) exist. However, these provers exhibit highly specialised and restricted spatial languages\footnote{Standard geometric constraint languages of approaches including \cite{chou1988mechanical,owen1991algebraic,hadas2000role} consist of points, lines, circles, ellipses, and coincidence, tangency, perpendicularity, parallelism, and numerical dimension constraints; note the absence of e.g. mereotopology and ``common-sense'' relative orientation relations \cite{Schultz-Bhatt-Borrmann-EGICE2014}.} and lack (a) the direct integration with more general AI methods and (b) the capacity for incorporating modular common-sense rules about space in an extensible domain- and context-specific manner.

The aims and contributions of the research presented in this paper are two-fold:

\begin{enumerate}

	\item to further develop a KR-centered declarative spatial reasoning paradigm such that spatial reasoning capabilities are available and accessible within AI programs and applications areas, and may be seamlessly integrated with other AI methods dealing with representation, reasoning, and learning about non-spatial aspects
	
	\item to demonstrate that in spite of high computational complexity in a general and domain-independent case, the power of analytic geometric ---in particular polynomial systems for encoding the semantics of spatial relations-- can be exploited by systematically utilising commonsense knowledge about spatial object and relationships at the qualitative level.
\end{enumerate}

We present a new algorithm that drastically improves analytic spatial reasoning performance within KR-based declarative spatial reasoning approaches by identifying and pruning spatial symmetries that form equivalence classes (based on affine transformations) at the qualitative spatial level. By exploiting symmetries our approach utilises powerful underlying, but computationally expensive, polynomial solvers in a significantly more effective manner. Our algorithm is simple to implement, and enables spatial reasoners to solve problems that are otherwise unsolvable using analytic or relation algebraic methods. We emphasise that our approach is independent of any particular polynomial constraint solver; it can be similarly applied over a range of solvers such as CLP(R), SMTs, and specialised geometric constraint solvers that have been integrated into a KR framework. 

In addition to AI / commonsense reasoning applications areas such as design, GIS, vision, robotics \citep{KR-2014-Bhatt,Bhatt-Wallgruen-2014,Bhatt2013-CMN,Bhatt-Schultz-Freksa:2013},  we also address application into automating support for proving the validity of theorems in mereotopology, orientation, shape, etc. (e.g. \cite{tarski1956general,borgo2013spheres}). Building on such foundational capabilities, another outreach is in the area of computer-aided learning systems in mathematics (e.g. at a high-school level). For instance, consider Proposition 9, Book I of Euclid's \emph{Elements}, where the task is to bisect an angle using only an unmarked ruler and collapsable compass. Once a student has developed what they believe to be a constructive proof, they can employ declarative spatial reasoners to formally verify that their construction applies to all possible inputs (i.e. all possible angles) and manipulate an interactive sketch that maintains the specified spatial relations (i.e. dynamic geometry \cite{hadas2000role}). A further area of interest is verifying the entries of composition tables that are used in relation algebraic qualitative spatial reasoning \cite{DBLP:conf/cade/RandellCC92}: given spatial objects $a,b,c \in U$, composition ``look up'' tables are indexed by pairs of (base) relations $R_{1_{ab}}$, $R_{2_{bc}}$ and return disjunctions of possible (base) relations $R_{3_{ac}}$. For each entry, the task is to prove $\exists a,b,c \in U \big( R_{1_{ab}} \wedge R_{2_{bc}} \wedge R_{3_{ac}} \big) $ for only those base relations $R_3$ in the entry's disjunction.


\section{Declarative Spatial Reasoning with CLP(QS)}\label{sec:dsr}
Declarative spatial reasoning denotes the ability of declarative programming frameworks in AI to handle \emph{spatial objects} and the \emph{spatial relationships} amongst them as \emph{native} entities, e.g., as is possible with concrete domains of Integers, Reals and Inequality relationships. The objective is to enable points, oriented points, directed line-segments, regions, and topological and orientation relationships amongst them as \emph{first-class} entities within declarative frameworks in AI \cite{bhatt-et-al-2011}.

\subsection{Examples of Declarative Spatial Reasoning with CLP(QS)}      
With a focus on spatial question-answering, the CLP(QS) spatial reasoning system \cite{bhatt-et-al-2011,schultz-bhatt-2012,DBLP:conf/ecai/SchultzB14} provides a practical manifestation of certain aspects of the declarative spatial reasoning paradigm in the context of constraint logic programming (CLP). \footnote{Spatial Reasoning (CLP(QS)). \texttt{www.spatial-reasoning.com}} CLP(QS) utilises a high-level language of spatial relations and commonsense knowledge about how various spatial domains behave. Such relations describe sets of object configurations, i.e. qualitative spatial relations such as \emph{coincident}, \emph{left of}, or \emph{partially overlapping}. Through this deep integration of spatial reasoning with KR-based frameworks, the long-term research agenda is to seamlessly provide spatial reasoning in other AI tasks such as planning, non-monotonic reasoning, and ontological reasoning \cite{bhatt-et-al-2011}. What follows is a brief illustration of the spatial Q/A capability supported by CLP(QS).

\textbf{EXAMPLE A}.\quad \emph{Massachusetts Comprehensive Assessment System (MCAS)}

\textbf{Grade 3 Mathematics (2009), Question 12.} \emph{Put a square and two right-angled triangles together to make a rectangle.} \emph{(1) Put the shapes $T_1, T_2, S$ illustrated in Figure \ref{fig:mcas-q12-shapes} together to make a rectangle.} \emph{(2) Put the shapes $T_1, T_2, S$ in Figure \ref{fig:mcas-q12-shapes}  together to make a quadrilateral that is not a rectangle}.

CLP(QS) represents right-angle triangles as illustrated in Figure \ref{fig:clpqs-ra-triangles}. Figures \ref{fig:mcsa-clpqs-a} and \ref{fig:mcsa-clpqs-b} present the CLP(QS) solutions.

 \begin{figure}
        
        \subfigure[CLP(QS) solution to arranging $T_1, T_2, S$ to form rectangles.]{
            \includegraphics[width=0.5\textwidth]{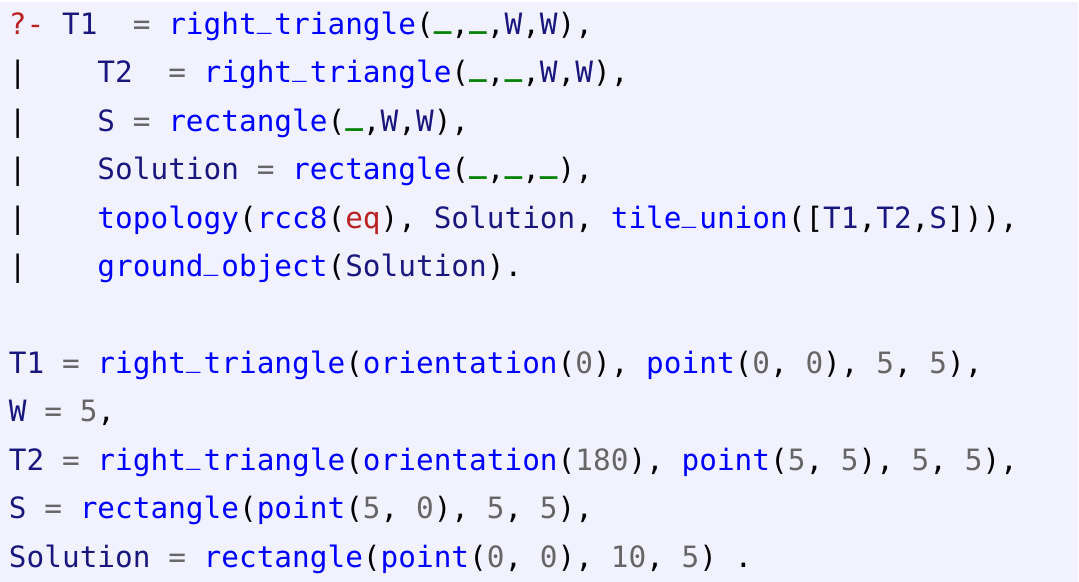}
            \label{fig:mcsa-clpqs-a}
        }
        \quad\quad\quad
        \subfigure[Representing right-angle triangle $T$ in CLP(QS).]{
            \includegraphics[width=0.4\textwidth]{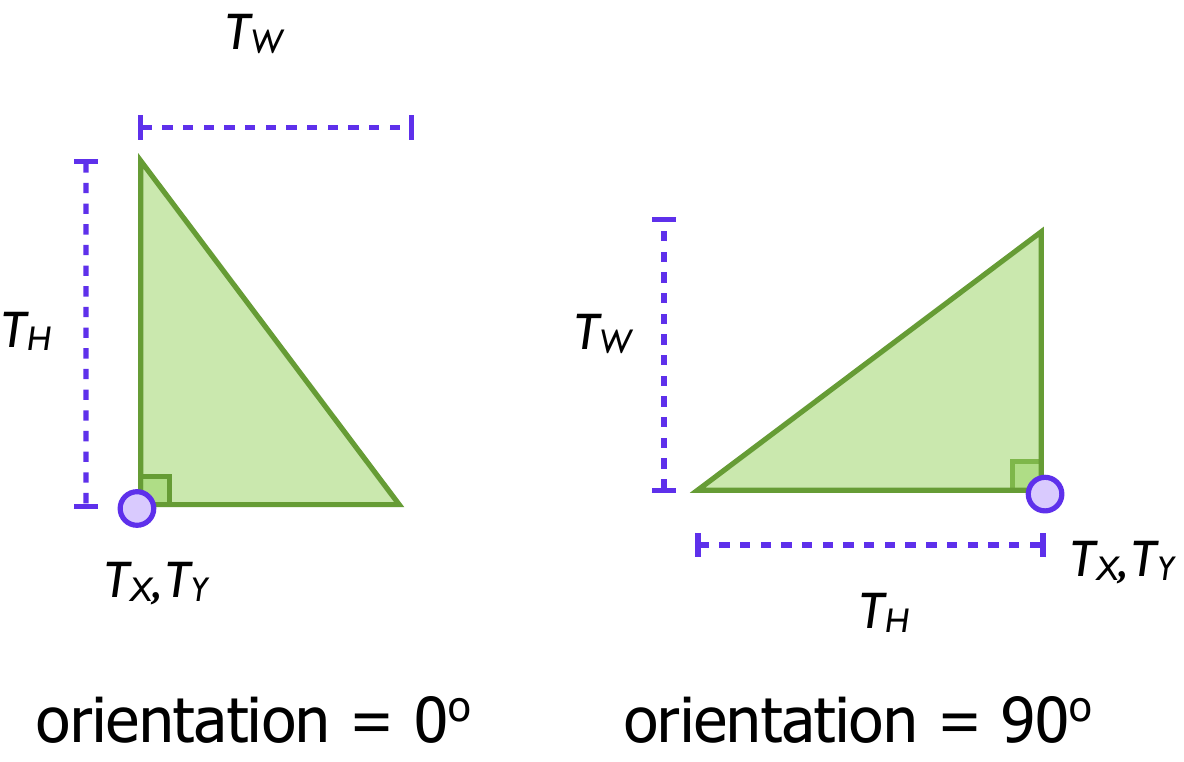}
            \label{fig:clpqs-ra-triangles}
        }

        \subfigure[CLP(QS) solution to arranging $T_1, T_2, S$ to form non-rectangular quadrilaterals.]{
            \includegraphics[width=0.5\textwidth]{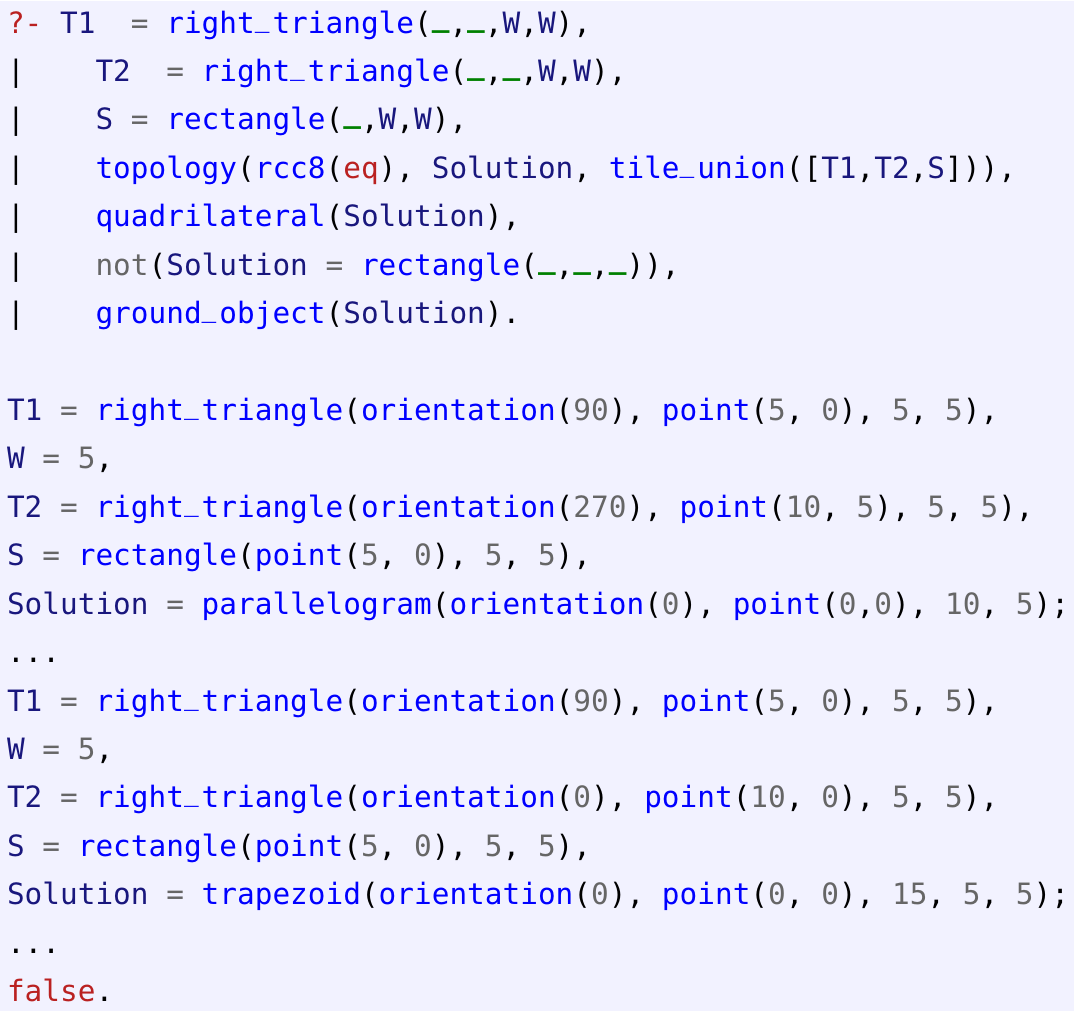}
            \label{fig:mcsa-clpqs-b}
        }
        \quad\quad\quad\quad\quad\quad
        \subfigure[Shapes $T_1, T_2, S$ to be arranged and CLP(QS) solutions.]{
            \includegraphics[width=0.25\textwidth]{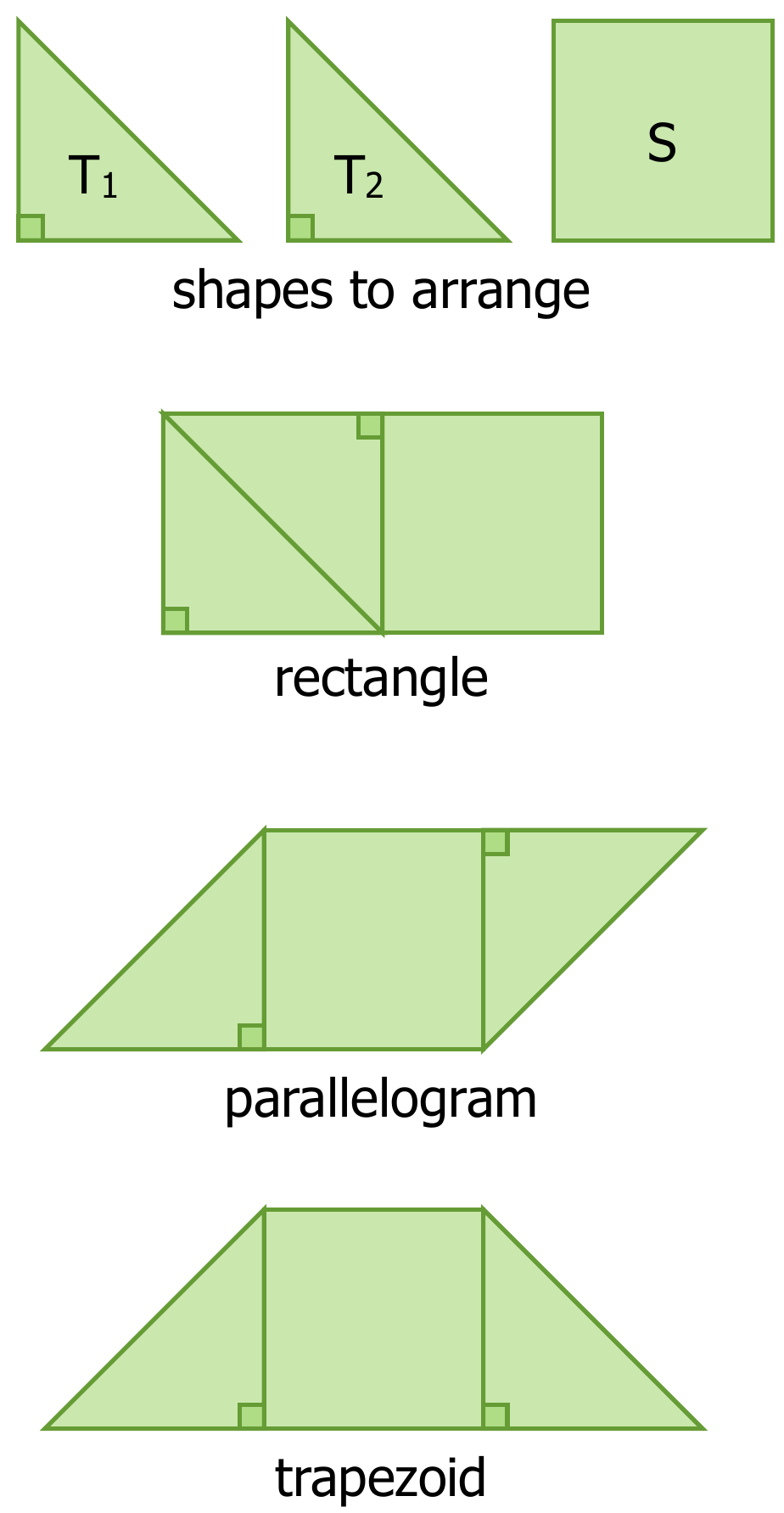}
            \label{fig:mcas-q12-shapes}
        }

                \caption{Using CLP(QS) to solve MCAS Grade 3 Mathematics Test questions (2009).}

    \end{figure}

\medskip

\textbf{Grade 3 Mathematics (2013), Question 17.} \emph{(1) How many copies of $T_1$ illustrated in Figure \ref{fig:mcas-q12-shapes} are needed to completely fill the region $R$ illustrated in Figure \ref{fig:mcsa-17-region} without any of them overlapping?}

As presented in Figure \ref{fig:mcsa-carlcarl}, CLP(QS) solves both the geometric definition and a variation where the dimensions of the rectangle and triangles are not given.

 \begin{figure}

        \subfigure[CLP(QS) solution to tiling a rectangular \newline region with triangle $T_1$.]{
            \includegraphics[width=0.5\textwidth]{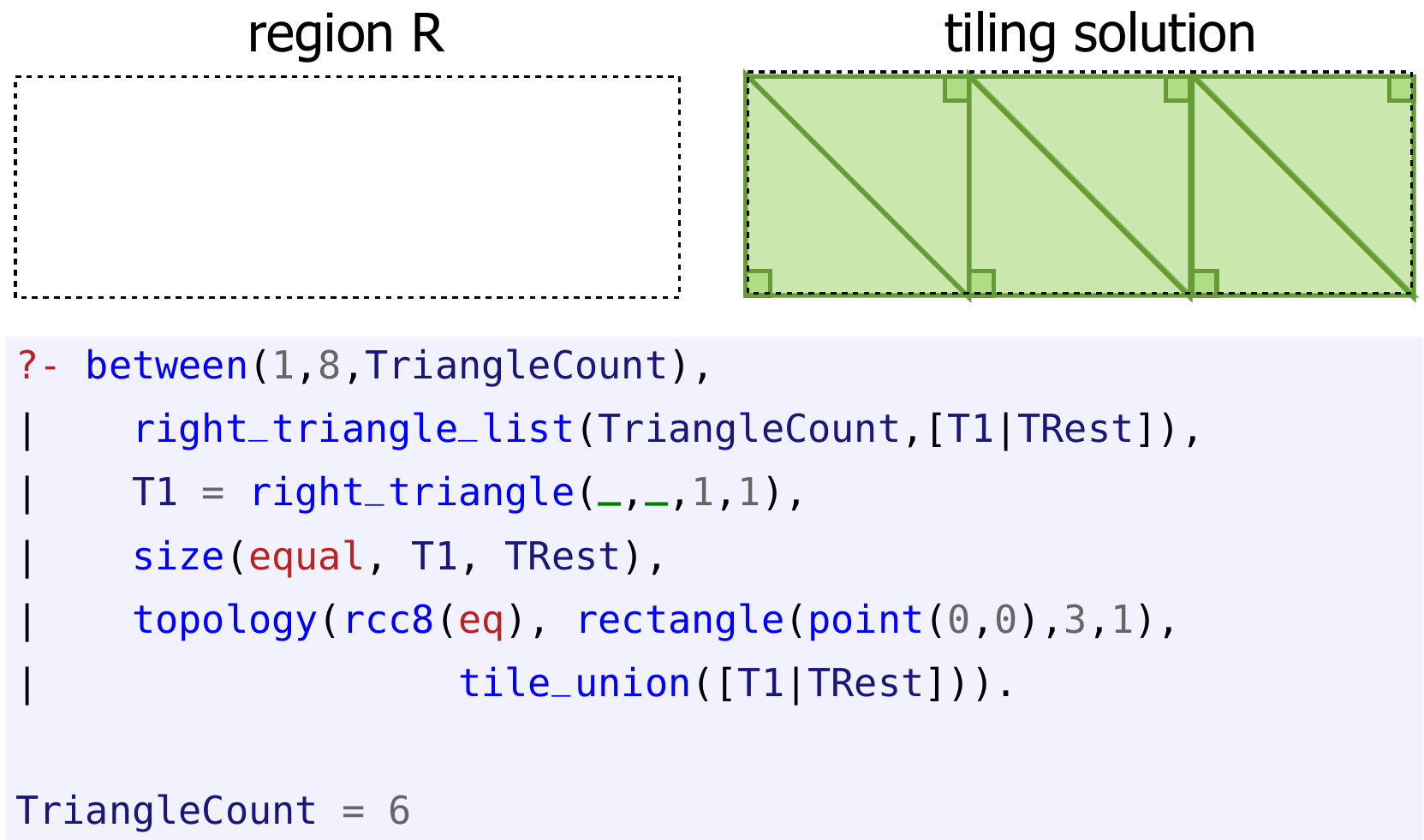}
            \label{fig:mcsa-17-region}
        }
        \subfigure[When no geometric information is given, CLP(QS) determines that the number of right-angle triangles must be even.]{
            \includegraphics[width=0.5\textwidth]{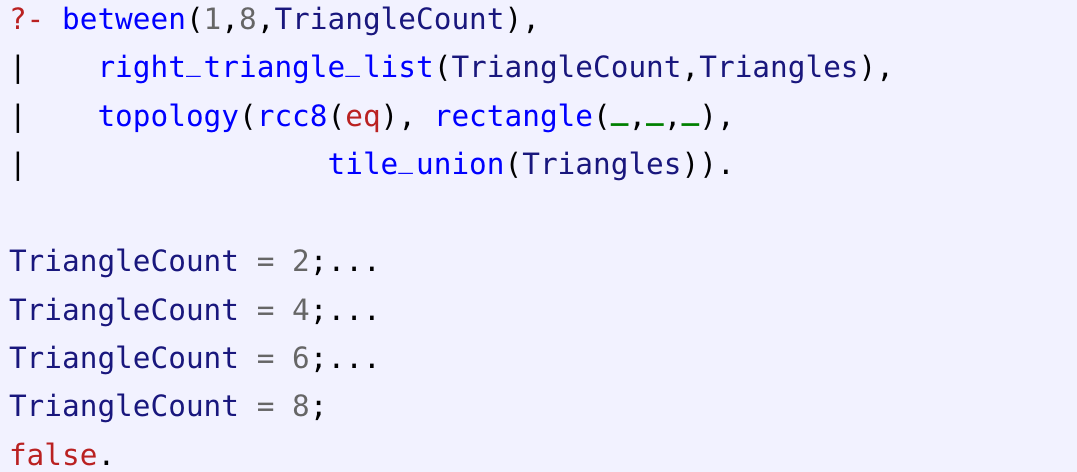}
        }

                \caption{Using CLP(QS) to solve MCAS Grade 3 Mathematics Test questions (2013).}
                \label{fig:mcsa-carlcarl}
    \end{figure}

 \begin{figure}
        \subfigure[CLP(QS) solution and its corresponding inferred \newline geometric constraints.]{
            \includegraphics[width=0.49\textwidth]{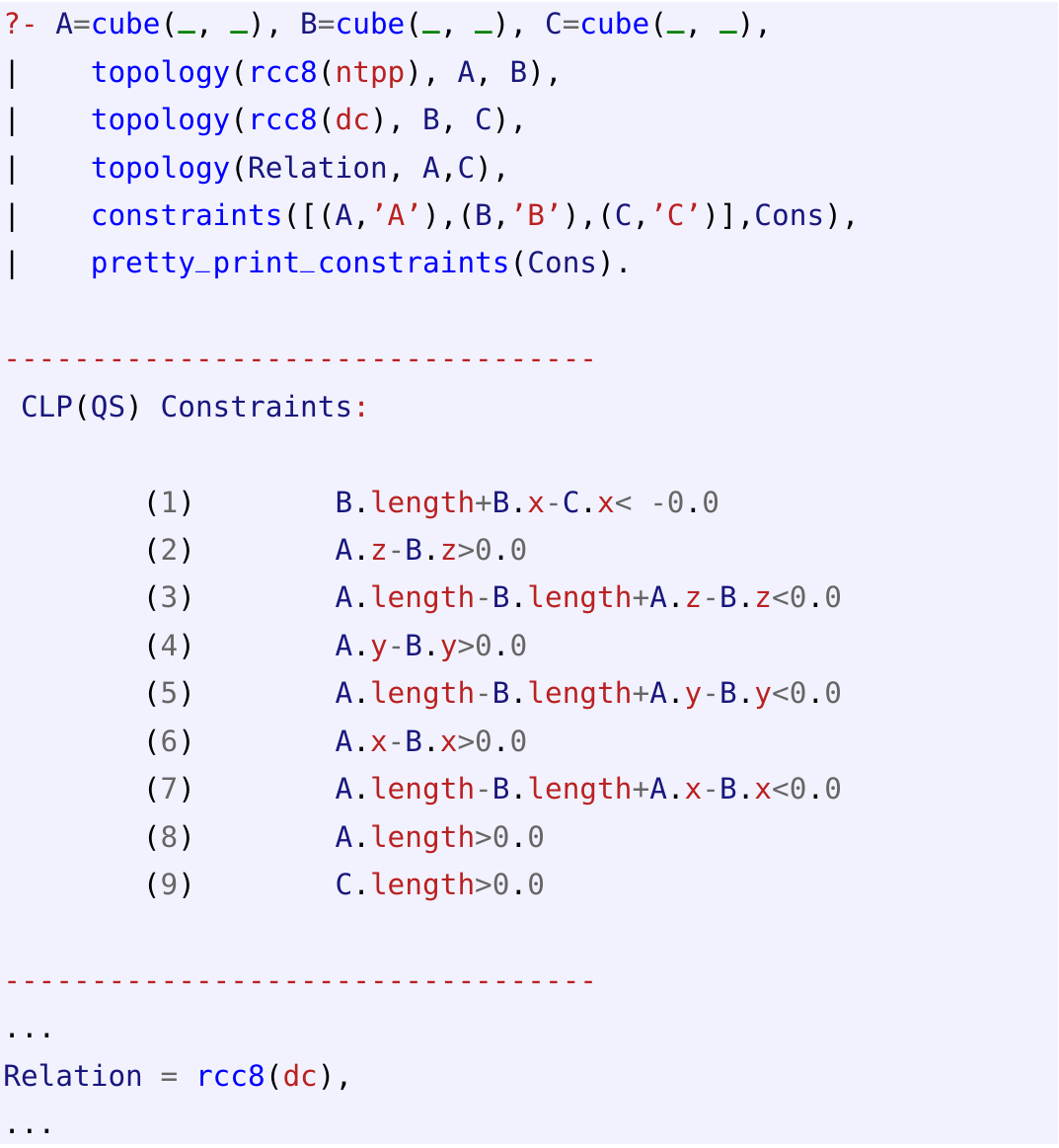}
            \label{fig:qsr-unknowns-code}
        }
        \subfigure[Geometric inequalities provided by CLP(QS) that correspond to the qualitative constraints.]{
            \includegraphics[width=0.5\textwidth]{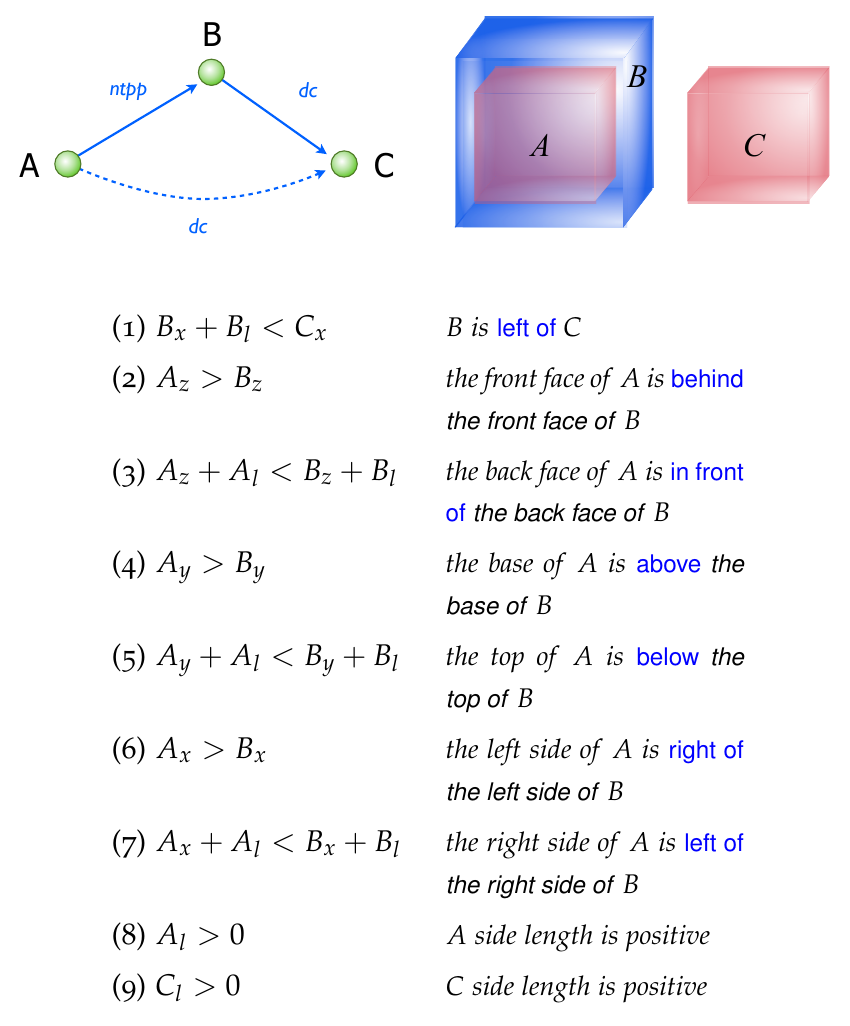}
        
        }

                \caption{Spatial reasoning about cubes $A$,$B$,$C$ with complete geometric unknowns.}
                \label{fig:qsr-unknowns}
    \end{figure}

\textbf{EXAMPLE B}.\quad \emph{Qualitative Spatial Reasoning with Complete Unknowns}

In this example CLP(QS) reasons about spatial objects based solely on given qualitative spatial relations, i.e. without any geometric information. 

\emph{Define three cubes $A, B, C$. Put $B$ inside $A$, and make $B$ disconnected from $C$. What spatial relations can possibly hold between $A$ and $C$?}

CLP(QS) determines that $A$ must be \emph{disconnected} from $C$ and provides the inferred corresponding geometric constraints, as illustrated in Figure \ref{fig:qsr-unknowns}.

\subsection{Analytical Geometry Foundations for Declarative Spatial Reasoning}
\label{sec:foundations} 

Analytic geometry methods parameterise classes of objects and encode spatial relations as systems of polynomial equations and inequalities \cite{chou1988mechanical}. For example, we can define a sphere as having a 3D centroid point $(x,y,z)$ and a radius $r$, where $x,y,z,r$ are reals. Two spheres $s_1, s_2$ \emph{externally connect} or \emph{touch} if
\begin{equation}
	(x_{s_1} - x_{s_2})^2 + (y_{s_1} - y_{s_2})^2 + (z_{s_1} - z_{s_2})^2 = (r_{s_1} + r_{s_2})^2
	\label{eqn:sphere-touch}
\end{equation}

If the system of polynomial constraints is satisfiable then the spatial constraint graph is consistent. Specifically, the system of polynomial (in)equalities over variables $X$ is satisfiable if there exists a real number assignment for each $x \in X$ such that the (in)equalities are \emph{true}. Partial geometric information (i.e. a combination of numerical and qualitative spatial information) is utilised by assigning the given real numerical values to the corresponding object parameters. Thus, we can integrate spatial reasoning and logic programming using \emph{Constraint Logic Programming} (CLP) \cite{jaffar1994constraint}; this system is called CLP over qualitative spatial domains. CLP($\mathcal{QS}$), provides a suitable framework for expressing and proving first-order spatial theorems.

Cylindrical Algebraic Decomposition (CAD) \cite{collins1975quantifier} is a prominent sound and complete algorithm for deciding satisfiability of a general system of polynomial constraints over reals and has time complexity $O(c_1^{c_2^n})$ in the number of free variables \cite{arnon1984cylindrical}. Thus, a key focus within analytic spatial reasoning has been methods for managing this inherent intractibility.\footnote{Important factors in determining the applicability of various analytic approaches are the degree of the polynomials (particularly the distinction between linear and nonlinear) and whether both equality and inequalities are permitted in the constraints.} More efficient refinements of the original CAD algorithm include partial CAD \cite{Collins1991299}. Symbolic methods for solving systems of multivariate \emph{equations} include the Gr\"{o}bner basis method \cite{buchberger2006bruno} and Wu's characteristic set method \cite{wenjun1984basic}. In the QUAD-CLP(R) system, the authors improve solving performance by using linear approximations of quadratic constraints and by identifying geometric equivalence classes \cite{pesant1994quad}. Ratschan employs pruning methods, also at the polynomial level, in the \emph{rsolve} system \cite{ratschan2002approximate,ratschan2006efficient}. 


Constructive and iterative (i.e. Newton and Quasi-Newton iteration) methods solve spatial reasoning problems by ``building'' a solution, i.e. by finding a configuration that satisfies the given constraints \cite{owen1991algebraic}. If a solution is found, then the solution itself is the proof that the system is consistent -- but what if a solution is not found within a given time frame? In general these methods are incomplete for spatial reasoning problems encoded as nonlinear equations and inequalities of arbitrary degree.\footnote{That is, constructive methods may fail in building a consistent solution, and iterative root finding methods may fail to converge.}

\section{Spatial Symmetries}
Information about \emph{objects} and their \emph{spatial relations} is formally expressed as a constraint graph $G = (N,E)$, where the nodes $N$ of the graph are spatial objects and the edges between nodes specify the relations between the objects. Objects belong to a \emph{domain}, e.g. points, lines, squares, and circles in 2D Euclidean space, and cuboids, vertically-extruded polygons, spheres, and cylinders in 3D Euclidean space. We denote the object domain of node $i$ as $U_i$ (spatial domains are typically infinite). A node may refer to a partially ground, or completely geometrically ground object, such that $U_i$ can be a proper subset of the full domain of that object type. Each element $i^\prime \in U_i$ is called an \emph{instance} of that object domain. A \emph{configuration} of objects is a set of instances $\{i^\prime_1,\dots,i^\prime_n\}$ of nodes $i_1,\dots,i_n$ respectively.

A binary relation $R_{i j}$ between nodes $i,j$ distinguishes a set of relative configurations of $i,j$; relation $R$ is said to \emph{hold} for those configurations, $R_{i j} \subseteq U_i \times U_j$. In general, an $n$-ary relation for $n \geq 1$ distinguishes a set of configurations between $n$ objects:  $R_{i_1,\dots,i_n} \subseteq U_{i_1} \times \dots \times U_{i_n}$. 

An edge between nodes $i,j$ is assigned a logical formula over relation symbols $R_1,\dots,R_m$ and logical operators $\vee, \wedge, \neg$. Given an interpretation $i^\prime, j^\prime$, the formula for edge $e$ is interpreted in the standard way, denoted $e(i^\prime,j^\prime)$: 

\begin{itemize}
	\item $R_1 \equiv (i^\prime,j^\prime) \in R_{1_{i j}}$
	\item $(R_1 \vee R_2) \equiv (i^\prime,j^\prime) \in R_{1_{i j}} \cup R_{2_{i j}}$
	\item $(R_1 \wedge R_2) \equiv (i^\prime,j^\prime) \in R_{1_{i j}} \cap R_{2_{i j}}$
	\item $(\neg R_1) \equiv (i^\prime,j^\prime) \in (U_i \times U_j) \setminus R_{1_{i j}}$
\end{itemize}

An edge between $i,j$ is \emph{satisfied} by a configuration $i^\prime, j^\prime$ if $e(i^\prime, j^\prime)$ is \emph{true} (this is generalised to $n$-ary relations). A spatial constraint graph $G=(N,E)$ is \emph{consistent} or \emph{satisfiable} if there exists a configuration $s$ of $N$ that satisfies all edges in $E$, denoted $G(s)$; this is referred to as the \emph{consistency task}. Graph $G^\prime$ is a \emph{consequence} of, or \emph{implied} by, $G$ if every spatial configuration that satisfies $G$ also satisfies $G^\prime$. This is the \emph{sufficiency task} (or \emph{entailment}) that we commonly apply to constructive proofs, where the task is to prove that objects and relations in $G$ are sufficient for ensuring that particular properties hold in $G^\prime$.

Given graph $G$, two key questions are (1) how to give meaning, or interpret, the spatial relations in $G$, and (2) how to efficiently determine consistency and produce instantiations of $G$. That is, we need to adopt a method for \emph{spatial reasoning}. 



\subsection{An Example of the Basic Concept}

A key insight is that spatial configurations form equivalence classes over qualitative relationships based on certain affine transformations. For example, consider the spatial task of determining whether five same-sized spheres can be mutually touching. Suppose we are given a specific numerically defined configuration of four mutually touching spheres as illustrated in Figure \ref{fig:spheres-eg-config}, and we prove that it is impossible to add an additional mutually touching sphere to this configuration. That is, let $s_1, \dots, s_4$ be unit spheres (radius $1$), centred on points $p_1 = (0,0,0)$, $p_2=(2,0,0)$, $ p_3=(1, \sqrt{3},0)$, $p_4=(1,\sqrt{\frac{1}{3}},\sqrt{\frac{8}{3}})$, respectively. According to Equation~\ref{eqn:sphere-touch}, $s_1, \dots, s_4$ are mutually touching. We prove that a fifth same-sized, mutually touching sphere cannot be added to this configuration by determining that the corresponding system of polynomial constraints  is unsatisfiable (the system consists of four constraints with three free variables $x_{s_5}, y_{s_5}, z_{s_5}$, by reapplying Equation \ref{eqn:sphere-touch} between $s_5$ and each other sphere, e.g. $s_1$ \emph{touches} $s_5$ is $ x_{s_5}^2 + y_{s_5}^2 + z_{s_5}^2 = 4$).

\begin{figure}[t]
\centering
\subfigure[initial configuration]{
	\includegraphics[height=0.2\columnwidth]{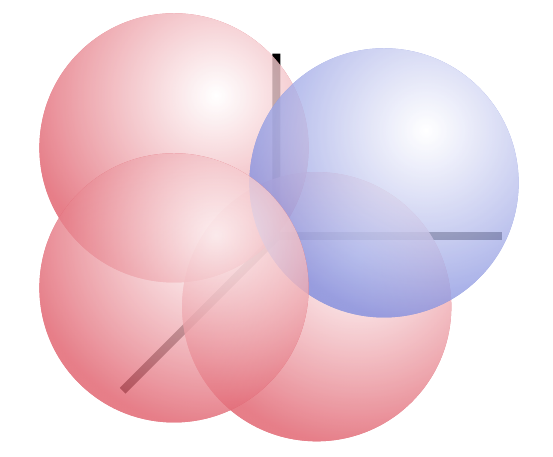}
	\label{fig:spheres-eg-config}
}\quad
\subfigure[rotation]{
	\includegraphics[height=0.2\columnwidth]{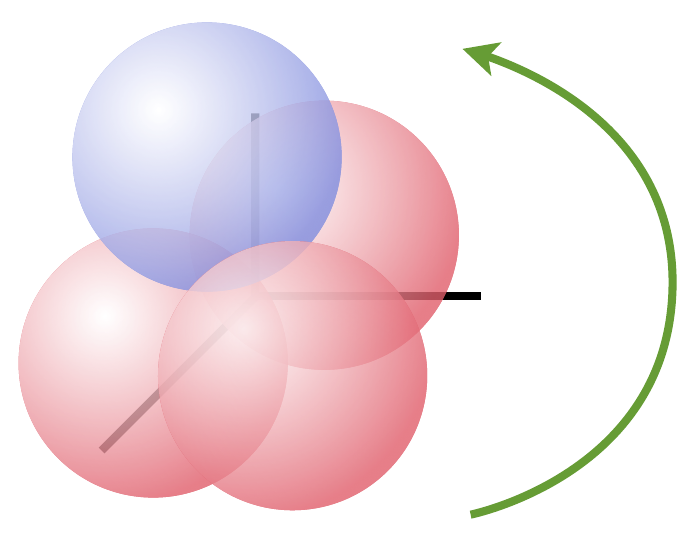}\
	\label{fig:spheres-eg-rotation}
}\quad
\subfigure[translation]{
	\includegraphics[height=0.2\columnwidth]{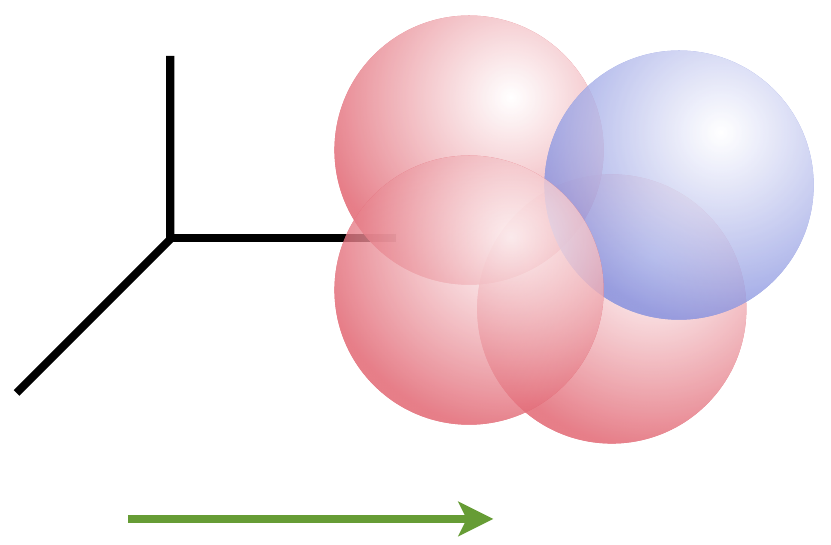}
	\label{fig:spheres-eg-translation}
}
\caption{Topological relations between four spheres maintained after various affine transformations.}
\label{fig:spheres-eg}
\end{figure}

Now consider that we apply an affine transformation to the original configuration such as rotation, translation, scaling, or reflection, as illustrated in Figures \ref{fig:spheres-eg-rotation} and \ref{fig:spheres-eg-translation}. After having applied the transformation, it is still impossible to add a fifth mutually touching sphere, because the relevant qualitative (topological) relations are \emph{preserved} under these transformations. Thus, when we proved that it was impossible to add a fifth same-sized mutually touching sphere to the original given configuration, in fact we proved it for a \emph{class} of configurations, specifically, the class of configurations that can be reached by applying an affine transformation to the original configuration. Now, when determining consistency of graphs of qualitative spatial relations, we are not given any specific spatial configurations to work with (i.e. complete absence of numerical information), and instead need to prove consistency over all possible configurations. 

The key is that, each time we ground and constrain variables, we are eliminating a \emph{spatial symmetry} from our partially defined configuration. If we maintain knowledge about symmetries that certain object types have (e.g. spheres have complete rotational symmetry) then we can judiciously ``trade'' symmetries for unbound variables in our polynomial encoding at a purely symbolic level. Importantly, rather than having to compute symmetries or undertake any complex symmetry detection procedure, we are instead \emph{building knowledge about space and spatial properties of objects into the spatial solver at a declarative level}. Thus, we are able to efficiently reason over an infinite set of possible configurations by incrementally pruning spatial symmetries based on commonsense knowledge about space, and this pruning is exploited by eliminating and constraining variables in the underlying polynomial encoding. 


\subsection{Theoretical Foundations for Symmetries}
\label{sec:theoretical-foundations}


Due to the parameterisation of objects, spatial configurations are embedded in $n$-dimensional Euclidean space $\Reals^n$ ($1 \leq n \leq 3$) with a fixed origin point. Let $V,W$ be Euclidean spaces in $\Reals^n$, each with an origin. Given vectors $x,y$ and constant $k$, a linear transformation $f$ is a mapping $V \rightarrow W$ such that 

$f(x +y) = f(x) + f(y)$ $~~~~~$ \emph{(additive)}

$f(kx) = k f(x)$ $~~~~~~~~~~~~~~~~$ \emph{(homogeneous)}

An affine transformation $f^\prime$ is a linear transformation composed with a translation. It is convenient to represent a linear transformation on vector $x$ as a left multiplication of a $d \times d$ real matrix $Q$, and translation as an addition of vector $t$, $f^\prime(x) = Qx + t$. We denote a transformation $T$ applied to a spatial configuration of objects $s$ as $Ts$.

We distinguish particular classes of transformations with respect to the qualitative spatial relationships that are preserved, for example, in $\Reals^2$ the following matrices represent rotation by $\theta$, uniform scaling by $k > 0$, and horizontal reflection, respectively: 

$\left( \begin{array}{rr}
cos(\theta) & -sin(\theta) \\
sin(\theta) & cos(\theta) \end{array} \right)$,
$\left( \begin{array}{rr}
k & 0 \\
0 & k \end{array} \right)$,
$\left( \begin{array}{rr}
-1 & 0 \\
0 & 1 \end{array} \right)$.


Given transformation $T$ we annotate it with its type $c \in C$, e.g. $C=\{$translate, rotate, scale, reflect$\}$  as $T^c$. Each spatial relation $R$ belongs to a class of relations in $\Rel$, such as \emph{topology, mereology, coincidence, relative orientation, distance}.  Let $\Sym$ be a function $\Sym: \Rel \rightarrow 2^C$ that represents the classes of transformations that preserve a given class of spatial relations. The $\Sym$ function is our mechanism for building knowledge about spatial symmetries into the spatial reasoning system. Let $\Rel_G$ be the set of classes of the spatial relations that are used in the spatial constraint graph $G$, and let $\Sym_G = \bigcap_{R \in \Rel_G} \Sym(R).$

The following formal Condition on $\Sym_G$ states that transformations (applied to the embedding space) define equivalence classes of configurations with respect to the consistency of spatial constraint graphs. When satisfied, this condition provides a theoretically sound foundation for symmetry pruning. 


\begin{myconditions}\label{cond:symcons}
Given spatial constraint graph $G$, configuration $s$, and affine transformation $T^c$ with $c \in \Sym_G$  then $G(s)$ is true if and only if $G(T^cs)$. 
\end{myconditions}

\subsection{Polynomial Encodings for Spatial Relations}

In this section we define a range of spatial domains and spatial relations with the corresponding polynomial encodings. While our method is applicable to a wide range of 2D and 3D spatial objects and qualitative relations, for brevity and clarity we primarily focus on a 2D spatial domain. Our method is readily applicable to other 2D and 3D spatial domains and qualitative relations, for example, as defined in \cite{pesant1994quad,bouhineau1996solving,pesant1999reasoning,bouhineau1999application,bhatt-et-al-2011,schultz-bhatt-2012,DBLP:conf/ecai/SchultzB14}.

\begin{itemize}
	\item a \emph{point} is a pair of reals $x,y$
	\item a \emph{line segment} is a pair of end points $p_1, p_2$ ($p_1 \neq p_2$)
	\item a \emph{rectangle} is a point $p$ representing the bottom left corner, a unit direction vector $v$ defining the orientation of the base of the rectangle, and a real width and height $w,h$ ($0 < w, 0 < h$); we can refer to the vertices of the rectangle: let $v^\prime = (-y_v, x_v)$ be $v$ rotated $90^o$ counter-clockwise,  then $p_1 = p = p_5, p_2 = wv + p_1, p_3 = wv + hv^\prime + p_1, p_4 = hv^\prime + p_1$
	\item a \emph{circle} is a centre point $p$ and a real radius $r$ ($0 < r$)
\end{itemize}

\medskip

\noindent We consider the following spatial relations: 

\smallskip

\noindent \emph{Relative Orientation.} Left, right, collinear orientation relations between points and segments, and parallel, perpendicular relations between segments \cite{leecomplexity-ecai2014}.

\smallskip

\noindent \emph{Coincidence.} Intersection between a point and a line, and a point and the boundary of a circle. Also whether the point is in the interior, outside or on the boundary of a region.

\smallskip

\noindent \emph{Mereology.} Part-whole relations between regions \cite{varzi1996parts}.

\bigskip 

\noindent Table \ref{tab:encodings} presents the corresponding polynomial encodings. Given three real variables $v, i, j$, let:
$$v \in [i,j] \equiv i \leq v \leq j \vee j \leq v \leq i.$$

\begin{table}[t]
\centering
\scriptsize
\begin{tabular}{|l|l|l|}
\hline
\textbf{\small Relation} & \textbf{\small Polynomial Encoding} \\
\hline
\hline
left of  & $(x_b - x_a) (y_p - y_a) > (y_b - y_a) (x_p - x_a)$ \\
(point $p$, segment $s_{ab}$) & \\
\hline
collinear  & $(x_b - x_a) (y_p - y_a) = (y_b - y_a) (x_p - x_a)$ \\
(point $p$, segment $s_{ab}$) & \\
\hline
right or collinear & $(x_b - x_a) (y_p - y_a) \leq (y_b - y_a) (x_p - x_a)$ \\
(point $p$, segment $s_{ab}$) & \\
\hline
parallel & $(y_b - y_a) (x_d - x_c) = (y_d - y_c) (x_b - x_a)$  \\
(segments $s_{ab}, s_{cd}$) & \\
\hline
coincident & $\Pred{collinear}(p,s_{ab}) \wedge x_p \in [x_a, x_b] \wedge y_p \in [y_a, y_b]$ \\
(point $p$, segment $s_{ab}$)  & \\
\hline
coincident  & $(x_c - x_p)^2 + (y_c - y_p)^2 = r_c^2$ \\
(point $p$, circle $c$) & \\
\hline
inside  & $(0 <  (p - p_{1_a}) \cdot v_a < w_a) \wedge$ \\
(point $p$, rectangle $a$) & $(0 < (p - p_{1_a}) \cdot v_a^\prime < h_a)$  \\
\hline
intersects & $(0 \leq  (p - p_{1_a}) \cdot v_a \leq w_a) \wedge$ \\
 (point $p$, rectangle $a$) & $ (0 \leq (p - p_{1_a}) \cdot v_a^\prime \leq h_a)$\\
\hline
boundary & $\Pred{intersects}(p,a) \wedge \neg \Pred{inside}(p,a)$ \\
(point $p$, rectangle $a$) & \\
\hline
outside & $\neg \Pred{intersects}(p,a)$ \\
(point $p$, rectangle $a$)  & \\
\hline
concentric & $\frac{1}{2}(p_{3_a} - p_{1_a}) + p_{1_a} = \frac{1}{2}(p_{3_b} - p_{1_b}) + p_{1_b}$ \\
(rectangles $a,b$) & \\
\hline
part of & $\bigwedge_{i=1\dots4} \Pred{intersects}(p_{i_a}, b)$\\
(rectangles $a,b$) & \\
\hline
proper part & $\neg  \Pred{equals}(a,b) \wedge \Pred{part\_of}(a,b)$\\
 (rectangles $a,b$) & \\
\hline
boundary part of & $\bigwedge_{i=1\dots4} \Pred{boundary}(p_{i_a}, b)$\\
(rectangles $a,b$)  & \\
\hline
discrete from  & $\bigvee_{i=1\dots4} \big( ~~ \Pred{right\_or\_collinear}(a,(p_{i_b}, p_{(i+1)_b}))$ \\
(rectangles $a,b$) & $~~~~~~~~~~~~~~~~~ \vee \Pred{right\_or\_collinear}(b,(p_{i_a}, p_{(i+1)_a})) \big)$ \\
\hline
partially overlaps & $\exists p_i \in \Reals^2 \big( \Pred{inside}(p_i,a) \wedge \Pred{inside}(p_i,b) \big) \wedge$   \\
(rectangles $a,b$)  & $\exists p_j \in \Reals^2 \big( \Pred{inside}(p_j,a) \wedge \Pred{outside}(p_j,b) \big) \wedge$  \\
& $\exists p_k \in \Reals^2 \big( \Pred{outside}(p_k,a) \wedge \Pred{inside}(p_k,b) \big)$  \\
\hline
\end{tabular}
\caption{\textit{\small Polynomial encodings of qualitative spatial relations.}}
\label{tab:encodings}
\end{table}

Determining whether a point is inside a rectangle is based on vector projection. Point $p$ is projected onto vector $v$ by taking the dot product, 
$$(x_p, y_p) \cdot (x_v, y_v) = x_p x_v + y_p y_v$$

Given point $a$ and rectangle $b$, we translate the point such that the bottom left corner of $b$ is at the origin, project $p_a$ on the base and side vectors of $b$, and check whether the projection lies within the width and height of the rectangle,
$$ 0 < (p_a - p_{1_b}) \cdot v_b < w_b$$
$$ 0 < (p_a - p_{1_b}) \cdot v_b^\prime < h_b$$

Convex regions $a,b$ are disconnected iff there is a hyperplane of separation, i.e. there exists a line $l$ such that $a$ and $b$ lie on different sides of $l$. This is the basis for determining the \emph{discrete from} relation between rectangles.

\subsection{Formalising Knowledge about Symmetries}

In this section we formally determine the qualitative spatial relations that are preserved by various affine transformations. A fundamental property of affine transformations is that they preserve (a) point coincidence (e.g. line intersections), (b) parallelism between straight lines, and (c) proportions of distances between points on parallel lines  \cite{martin1982transformation}. For example, consider the configuration of points $p_a,p_b,p_c,p_d$ and lines $l_1, l_2, l_3$ in Figure \ref{fig:affine-eg}: (a) we cannot introduce new points of coincidence between lines by applying transformations such as translation, scaling, reflection, and rotation. Conversely, if two lines intersect, then they will still intersect after these transformations; (b) lines $l_1,l_2$ are parallel before and after the transformations; lines $l_1, l_3$ are always non-parallel; (c) the ratio of distances between collinear points $p_a, p_b, p_c$ is maintained; formally, let $s_{ij}$ be the segment between points $p_i, p_j$ and let $|s_{ij}|$ be the length of the segment. Then the ratio $\frac{|s_{ab}|}{|s_{bc}|}$ in Figure \ref{fig:affine-eg} is the same before and after the transformations.




\begin{figure}[t]
\centering
\subfigure[\scriptsize configuration of points and lines]{
	\includegraphics[height=0.17\columnwidth]{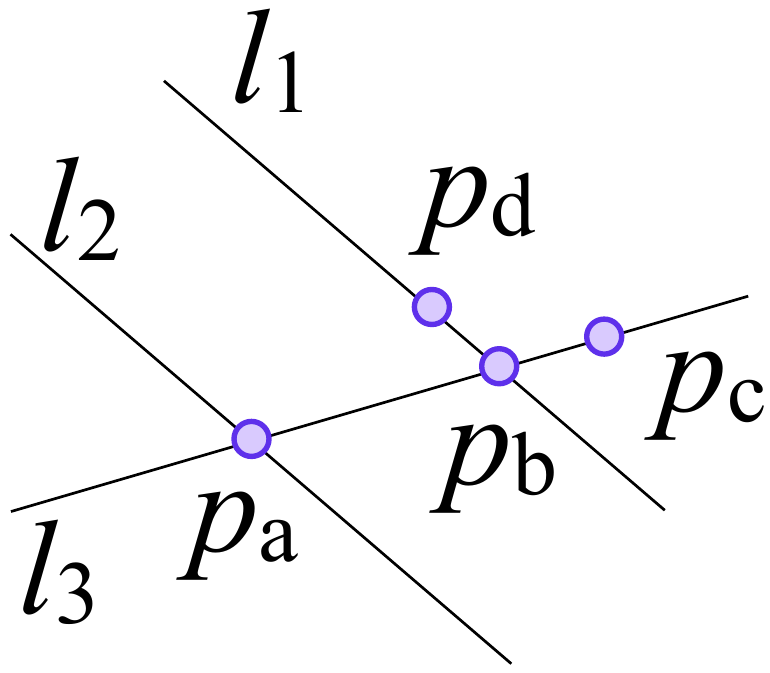}
	\label{fig:affine-eg-1}
} \quad
\subfigure[\scriptsize vertical reflection]{
	\includegraphics[height=0.17\columnwidth]{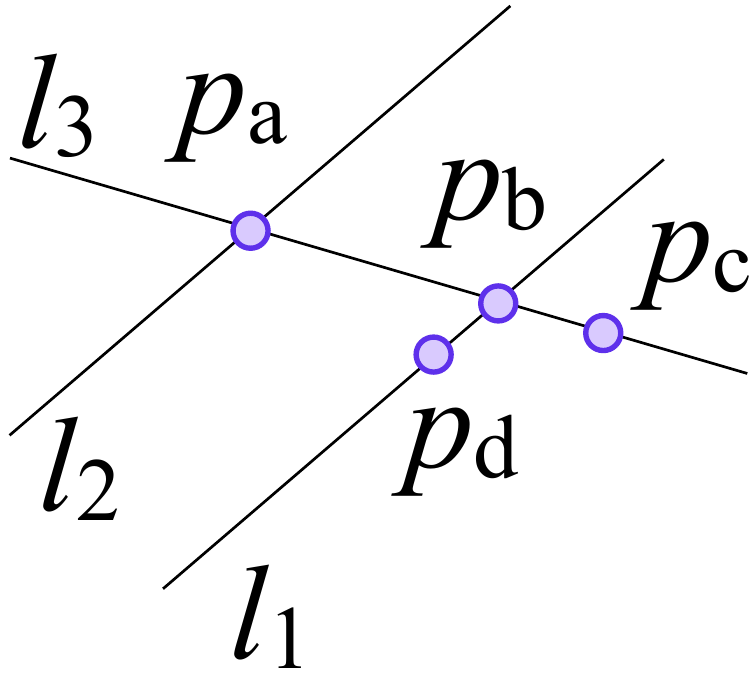}
	\label{fig:affine-eg-2}
} \quad
\subfigure[\scriptsize $x$-scaling]{
	\includegraphics[height=0.17\columnwidth]{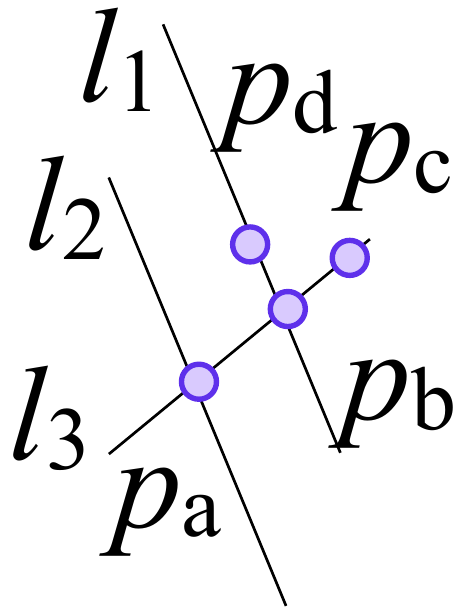}
}\quad
\subfigure[\scriptsize rotation]{
	\includegraphics[height=0.17\columnwidth]{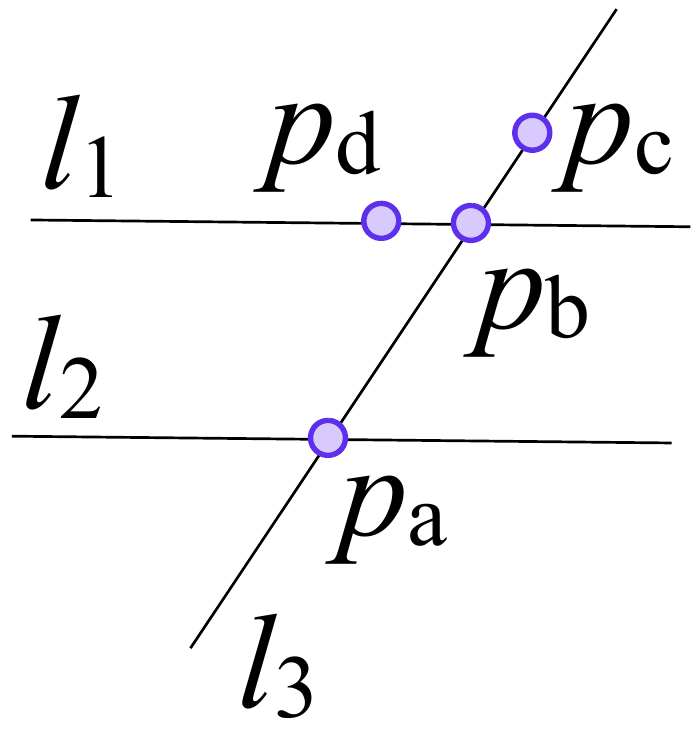}
}
\caption{Affine transformations preserve point coincidence, parallelism, and ratios of distances along parallel lines.}
\label{fig:affine-eg}
\end{figure}


Based on these properties we can determine the transformations that preserve various qualitative spatial relations.\footnote{The properties of affine transformations and the geometric objects that they preserve are well understood; further information is readily available in introductory texts such as \cite{martin1982transformation}. Our key contribution is formalising and exploiting this spatial knowledge as modular and extensible common-sense rules in intelligent knowledge-based spatial assistance systems.}





\smallskip


\begin{mytheorems}\label{th:sym}

The following qualitative spatial relations are preserved under translation, scale, rotation, and reflection (applied to the embedding space): topology, mereology, coincidence, collinearity, line parallelism.

\end{mytheorems}

\begin{proof}
By definition, affine transformations preserve \emph{parallelism} with respect to qualitative line orientation, and point \emph{coincidence}. Due to preservation of point coincidence and proportions of collinear distances by affine transformations, it follows that mereological \emph{part of} and topological \emph{contact} relations between regions are preserved, i.e. if a mereological or topological relation changes between regions $a,b$, then by definition there exists a point $p$ coincident with $a$ such that the coincidence relation between $p$ and $b$ has changed; but this cannot occur as point coincidence is maintained with affine transformations by definition, therefore mereological and topological relations are also maintained. 
\end{proof}

The interaction between spatial relations and transformations is richer than we have space to elaborate on here, i.e. not all qualitative spatial relations are preserved under all affine transformations; orientation is not preserved under reflection (e.g. Figure \ref{fig:affine-eg-2} gives a counter example), distance is not preserved under non-uniform scaling. To summarise, we formalise the following knowledge as modular commonsense rules in CLP(QS): point-coincidence, line parallelism, topological and mereological relations are preserved with all affine transformations. Relative orientation changes with reflection, and qualitative distances and perpendicularity change with non-uniform scaling. Spheres, circles, and rectangles are not preserved with non-uniform scaling, with the exception of axis-aligned bounding boxes.

\subsubsection{``Trading'' Transformations}

Symmetries are used to eliminate object variables. As a metaphor, unbound variables are replaced by constant values in ``exchange'' for transformations. We start with a set of transformations that can be applied to a configuration: translation, scaling, arbitrary rotation, and horizontal and vertical reflection. We can then ``trade'' each transformation for an elimination of variables. Each transformation can only be ``spent'' once. Theorem\;\ref{thr:prune-case} presents an instance of such a pruning case. 

Table\;\ref{tab:pruning} presents a variety of different pruning cases for position variables and the associated combination of transformations, as illustrated in Figure \ref{fig:position-pruning}.\footnote{All cases have been verified using Reduce as presented in Theorem\;\ref{thr:prune-case}.} Some cases require more than one distinct set of parameter restrictions to cover the set of all position variables due to point coincidence being preserved by affine transformations. For example, consider case (f): all pairs of points $p_1, p_2$ can be transformed into any other pair of points $p_i, p_j$ by translation, rotation, and scaling, iff $p_1 = p_2 \leftrightarrow p_i = p_j$. Thus, to cover all pairs of possible points, we need to consider two distinct parameter restrictions: $p_i = p_j$ and $p_i \neq p_j$; we refer to these as \emph{subcases}.

Many further pruning cases are identifiable. For example, a version of case $(i)$ can be defined without reflection by requiring more sub-cases where $c_4 > c_2$ and $c_4 < c_2$. Case $(i)$ can be extended so that all six coordinates of three points are grounded if we also ``exchange'' the \emph{skew} transformation (e.g. applicable to object domains like triangles or points).


\begin{mytheorems}\label{thr:prune-case}
Any pair of object position variables $(x_1,y_1), (x_2,y_2)$ can be transformed into any given position constants $(c_1,c_2), (c_3,c_2)$ such that $(c_1 = c_3 \leftrightarrow (x_1 = x_2 \wedge y1=y2))$ by applying: an $xy$-translation, a rotation about the origin in the range $(0,2\pi)$, and an $x$-scale.
\end{mytheorems}

\begin{proof}
The corresponding expression has been verified using the \emph{Reduce} system (\emph{Redlog} quantifier elimination) \cite{Dolzmann2006}; all variables are quantified over reals.

\noindent {\footnotesize
$\forall c_1 \forall c_2  \forall c_3  \forall x_1 \forall y_1 \forall x_2 \forall y_2 \\ ~~ (c_1 = c_3 \leftrightarrow (x_1 = x_2 \wedge y1=y2)) \leftrightarrow \exists t_x \exists t_y \exists d_x \exists d_y  \exists s_x \big( \\ ~~ (0 < s_x) \wedge (d_x^2 + d_y^2 = 1) \wedge \\ ~ \Let S = \VectorC{s_x ~~ 0}{~~ 0 ~~ 1} \wedge \Let R = \VectorC{d_x -d_y}{d_y ~~~ d_x} \wedge \Let T = \VectorC{t_x}{t_y} \wedge \\ ~ \VectorC{c_1}{c_2} = S R \VectorC{x_1}{y_1} + T \wedge \VectorC{c_3}{c_2} = S R \VectorC{x_2}{y_2} + T \big) \equiv \top$
}
\end{proof}

We can use this pruning case on any spatial constraint graph $G$ where the graph's spatial relations are preserved by translation, rotation, and scaling. Given graph $G=(N,E)$, the following Algorithm applies the pruning case, with selected constants $c_1=0$, $c_2=0$, and $c_3=1$ or $c_3=0$:

\begin{enumerate}
	\item select object position variables $p_1, p_2$ from nodes in $N$
	\item copy $G$ to create $G_1, G_2$
	\item in $G_1$ set $p_1 = (0,0), p_2=(1,0)$ $~$ \emph{(case $c_1 \neq c_3$)}
	\item in $G_2$ set $p_1 = (0,0), p_2=(0,0)$ $~$ \emph{(case $c_1 = c_3$)}
	\item if the task is:
	
	\begin{enumerate}
		\item \emph{consistency} of $G$ then solve $\bigvee_{i=1}^2 \exists s ~ G_i (s)$
		\item \emph{sufficiency}, $G \rightarrow G^\prime$, then solve \\ $\bigwedge_{i=1}^2 \neg \exists s ( G_i (s) \wedge \neg G^\prime(s))$
	\end{enumerate}
	
\end{enumerate}

In Step 1 any pair of objects can be selected for which their position variables will be grounded; we also employ policies that target computationally costly subgraphs (for example, pairs of non-equal circles that share a boundary point are often good candidates for this pruning case). Having eliminated free variables from the system of polynomial constraints, the constraints are significantly more simple to solve. Due to the double exponential complexity $O(c_1^{c_2^n})$ reducing $n$ has a significant impact on performance; the system may even collapse from nonlinear constraints to linear (solvable in $O(c^n)$) or constants.

\begin{figure}[t]
\centering
\subfigure[]{
	\includegraphics[height=0.18\columnwidth]{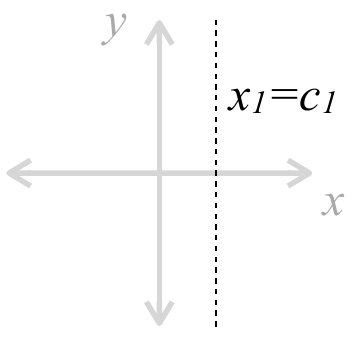}
} \quad
\subfigure[]{
	\includegraphics[height=0.18\columnwidth]{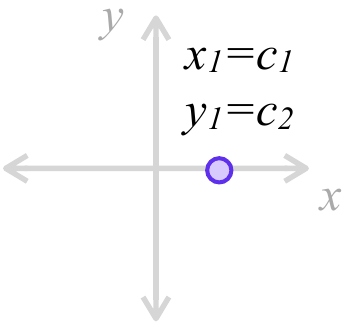}
} \quad
\subfigure[]{
	\includegraphics[height=0.18\columnwidth]{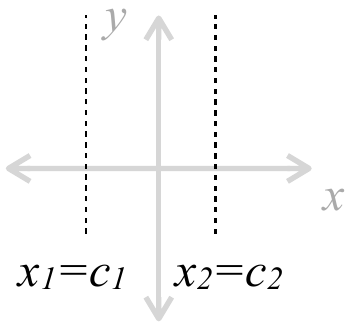}
}

\subfigure[]{
	\includegraphics[height=0.18\columnwidth]{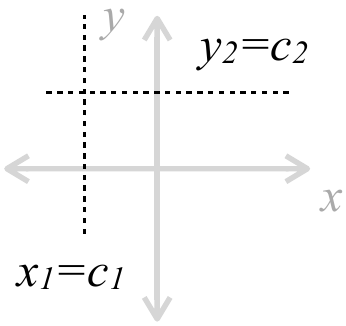}
} \quad
\subfigure[]{
	\includegraphics[height=0.18\columnwidth]{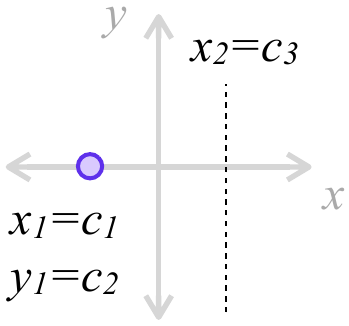}
} \quad
\subfigure[]{
	\includegraphics[height=0.18\columnwidth]{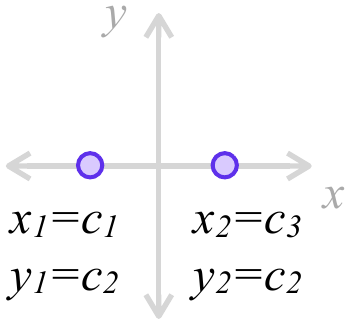}
}

\subfigure[]{
	\includegraphics[height=0.18\columnwidth]{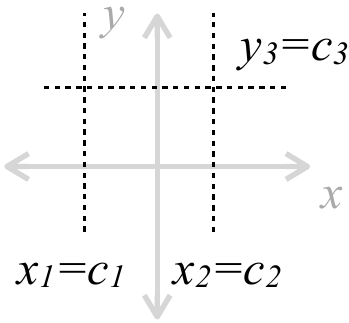}
} \quad
\subfigure[]{
	\includegraphics[height=0.18\columnwidth]{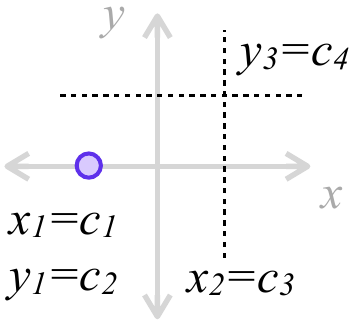}
} \quad
\subfigure[]{
	\includegraphics[height=0.18\columnwidth]{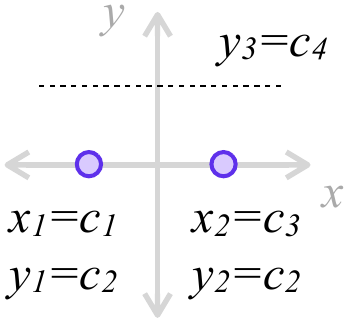}
}
\caption{Cases for pruning position parameters.}
\label{fig:position-pruning}
\end{figure}

\begin{table}
\centering
\scriptsize
\begin{tabular}{|c||l|l|}
\hline
\textbf{\small Case} & \textbf{\small Parameter restrictions} & \textbf{\small Traded transformations} \\
\hline
\hline
a & $x_1=c_1$ & $x$-translate \\
b & $x_1=c_1, y_1=c_2$ & $xy$-translate \\
\hline
c & $x_1=c_1, x_2=c_2, \casei c_1 \neq c_2 \caseii c_1 = c_2$ & $x$-translate, rotate $\pi$, $x$-scale \\
d & $x_1=c_1,y_2=c_2$ & $xy$-translate \\
e & $x_1=c_1, y_1=c_2,x_2=c_3,$ & $xy$-translate, rotate $\pi$, $x$-scale \\
& $\casei c_1 \neq c_3 \caseii c_1 = c_3$ & \\
f & $x_1=c_1, y_1=c_2, x_2=c_3, y_2=c_2, $ & $xy$-translate, rotate $(0,2\pi)$, $x$-scale \\
& $\casei c_1 \neq c_3 \caseii c_1 = c_3$ & \\
\hline
g & $x_1=c_1, x_2=c_2, y_3=c_3$ & $xy$-translate, rotate $\pi$, $x$-scale \\
& $\casei c_1 \neq c_2 \caseii c_1 = c_2$ & \\
h $*$ & $x_1=c_1, y_1=c_2,x_2=c_3, y_3=c_4$ & $xy$-translate, rotate $\pi$, $xy$-scale, \\
& $\casei \;\; c_1 \neq c_3 \wedge  c_2 \neq c_4 \caseii c_1 = c_3 \wedge c2 \neq c_4$ & $y$-reflect\\
& $\caseiii c_1 \neq c_3 \wedge  c_2 = c_4 \caseiv c_1 = c_3 \wedge c2 = c_4$ & \\
i $*$ & $x_1=c_1, y_1=c_2, x_2=c_3, y_2=c_2, y_3=c_4$ & $xy$-translate, rotate $(0,2\pi)$, $xy$-scale, \\
& $\casei \;\; c_1 \neq c_3 \wedge  c_2 \neq c_4 \caseii c_1 \neq c_3 \wedge  c_2 = c_4$ & $y$-reflect \\
& $ \caseiii c_1 = c_3 \wedge c_2 = c_4$ & \\
\hline
\end{tabular}
 \caption{\textit{\small Cases for pruning parameters for one position point (a,b), two position points (c-f), three position points (g-i). Cases marked with $*$ require arbitrary scaling (i.e. both uniform and non-uniform).}}
\label{tab:pruning}
\end{table}

\subsection{Combining Symmetry Pruning with Graph Decomposition}

In certain cases, spatial constraint graphs can be decomposed into subgraphs that can be solved independently. For example, subgraphs $G_1, G_2$ can be independently solved if all objects in subgraph $G_1$ are either:

\begin{itemize}
	\item  \emph{disconnected} from all objects in subgraph $G_2$;
	\item  a \emph{proper part} of some object in $G_2$;
	\item \emph{left of} some segment in $G_2$;
	\item only related by \emph{relative size} to some object in $G_2$, and so on.
\end{itemize}

In such cases we can reapply spatial symmetry pruning in each independent sub-graph; this commonsense spatial knowledge is modularly formalised within CLP(QS). For example, consider Proposition 22 of Book I of Euclid's Elements (Figure \ref{fig:construct-triangle}): 

\medskip

\emph{Constructing a triangle from three segments.} Given three line segments $l_{ab}$, $l_{cd}$, $l_{ef}$, draw a line through four collinear points $p_1, \dots, p_4$ such that $|(p_1,p_2)|=|l_{ab}|$, $|(p_2,p_3)|=|l_{cd}|$, $|(p_3,p_4)|=|l_{ef}|$. Draw circle $c_a$ centred on $p_2$, coincident with $p_1$. Draw circle $c_b$ centred on $p_3$ coincident with $p_4$. Draw $p_5$ coincident with $c_a$ and $c_b$. The triangle $p_2, p_3, p_5$ has side lengths such that $|(p_2,p_3)|=|l_{cd}|$, $|(p_3,p_5)|=|l_{ef}|$, $|(p_5,p_2)|=|l_{ab}|$.

\medskip

In this example, the three segments $l_{ab}, l_{cd}, l_{ef}$ and the remaining objects are only related by the distances between their end points. That is, the relative position and orientation of $l_{ab}, l_{cd}, l_{ef}$ is not relevant to the consistency of the spatial graph; we only need to explore all combinations of segment \emph{lengths}. Thus the solver decomposes the graph into four sub-graphs: (1) $l_{ab}$ (2) $l_{cd}$ (3) $l_{ef}$, and (4) $p_1, \dots, p_5, c_a, c_b$. In subgraphs (1),(2),(3) it ``trades'' translation and rotation to ground $p_a=p_c=p_e=(0,0)$, and $y_b=y_d=y_f=0$ and keeps $x$-scale to cover all possible combinations of segment lengths, i.e. $x_b, x_d, x_f$ are free variables. In subgraph (4) CLP(QS) applies the pruning case of Theorem \ref{thr:prune-case} by grounding $p_1, p_4$.

\includegraphics[]{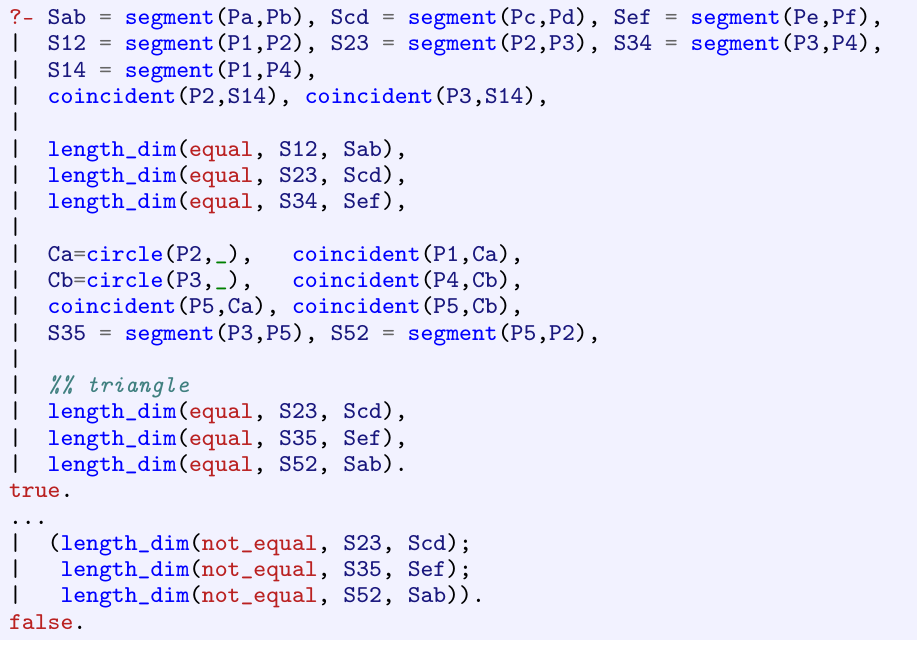}

\begin{figure}[h]
\centering
\includegraphics[width=0.45\columnwidth]{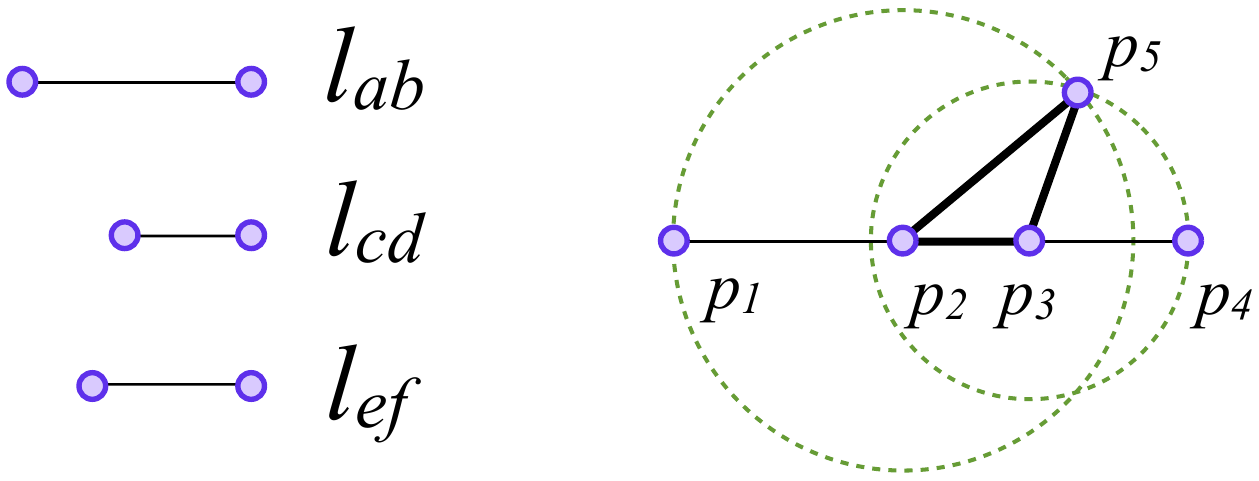}
\caption{Constructing a triangle by decomposing the spatial constraint graph.}
\label{fig:construct-triangle}
\end{figure}

\section{Application-Driven Use Cases}

We present problem instances in the classes of \emph{mereology}, \emph{ruler and compass}, and \emph{contact}. Table \ref{tab:experiments} presents the experiment time results of CLP(QS) using symmetry pruning compared with existing systems: \emph{z3} SMT solver, \emph{Redlog} real quantifier elimination (in the \emph{Reduce} computer algebra system) \cite{Dolzmann2006}, and the relation algebraic qualitative spatial reasoners GQR \cite{gantner2008gqr} and SparQ \cite{cosy:sparq-sc06}. CLP(QS) uses \emph{z3} to solve polynomial constraints (after our pruning), thus z3 is the most direct comparison. Experiments were run on a MacBookPro, OS X 10.8.5, 2.6 GHz Intel Core i7. The empirical results show that no other spatial reasoning system exists (to the best of our knowledge) that can solve the range of problems presented in this section, and in cases where solvers are applicable, CLP(QS) with spatial pruning solves those problems significantly faster than other systems.

\begin{table}
\centering
\scriptsize
\begin{tabular}{||l|r|r|r|r|r|r|| }
\hline
\textbf{\small Problem} & CLP(QS) & z3 & Redlog & GQR & SparQ \\
\hline
\hline
Aligned Concentric  & 6.831 & 47.651 & \timeout & \notapp & \notapp \\
\hline
Boundary Concentric & 2.036  & \timeout  & \timeout & \notapp & \notapp \\
\hline
Mereologically Concentric & 0.105 & 0.373  & \timeout & \notapp & \notapp \\
\hline
Angle Bisector & 0.931 & \timeout  & \timeout & \notapp & \notapp \\
\hline
Sphere Contact & 0.004 & \timeout & \timeout & \fail & \fail\\
\hline
\end{tabular}
 \caption{\textit{\small  Time (in seconds) to solve benchmark problems using CLP(QS) with pruning compared to z3 SMT solver, Redlog (Reduce) quantifier elimination, and GQR and SparQ relation algebraic solvers. \emph{Time out} was issued after a running time of 10 minutes. Failure (\fail) indicates that the incorrect result was given. Not applicable (\notapp) indicates that the problem could not be expressed using the given system.}}
\label{tab:experiments}
\end{table}

\subsection{Spatial Theorem Proving: Geometry of Solids} 

%
%

\noindent Tarski \cite{tarski1956general} shows that a geometric point can be defined by a language of mereological relations over spheres. The idea is to distinguish when spheres are concentric, and to define a geometric point as the point of convergence. Borgo \cite{borgo2013spheres} shows that this can be accomplished with a language of mereology over \emph{hypercubes}. We will use CLP(QS) to prove that the definitions are sound for rotatable squares.



As a preliminary we need to determine whether the intersection of two squares is non-square (Figure \ref{fig:non-square-intersect}) \cite{DBLP:conf/ecai/SchultzB14}. Given two squares $a,b$, the intersection is non-square if $a$ \emph{partially overlaps} $b$ (Table \ref{tab:encodings}) and either (a) $a$ and $b$ are not aligned, $x_{v_a} \neq x_{v_b}$ or (b) the width and height of the intersection are not equal, $w_I \neq h_I$, such that
$$ w_I = \Pred{min}(v \cdot p_{2_a}, v \cdot p_{2_b}) - \Pred{max}(v \cdot p_{1_a}, v \cdot p_{1_b})  $$ 
$$ h_I = \Pred{min}(v^\prime \cdot p_{4_a}, v^\prime \cdot p_{4_b}) - \Pred{max}(v^\prime \cdot p_{1_a}, v^\prime \cdot p_{1_b})  $$

%

\begin{figure}[t]
\centering
\subfigure[$A,B$ are not aligned.]{\includegraphics[width=0.25\columnwidth]{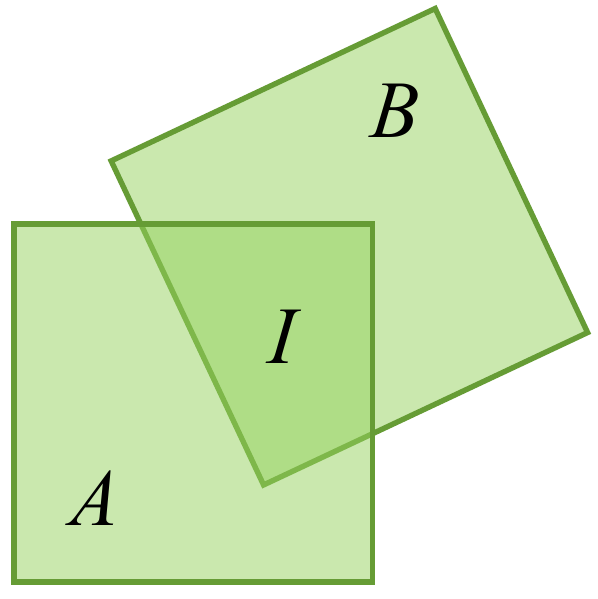}}\quad
\subfigure[$h_I \neq w_I$]{\includegraphics[width=0.32\columnwidth]{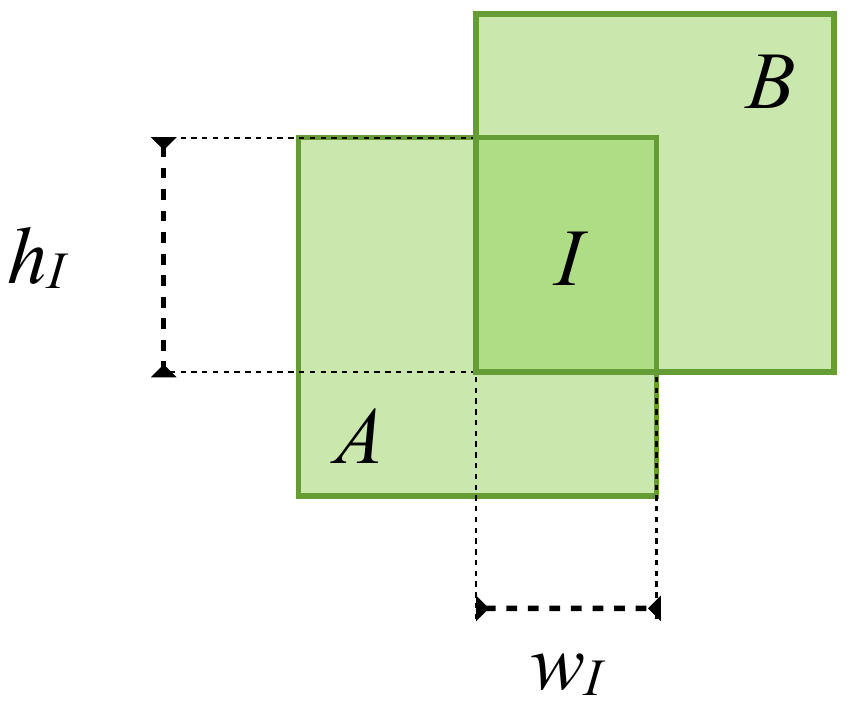}}\quad
\caption{Intersection $I$ of squares $A,B$ is non-square.}
\label{fig:non-square-intersect}
\end{figure}

\emph{Aligned Concentric.} Two squares $A,B$ are aligned and concentric if: $A$ is \emph{part of} $B$ and there does not exist a square $P$ such that (a) $P$ is covertex with $B$, and (b) the intersection of $P$ and $A$ is not a square (Figure \ref{fig:aligned-concentric}).

\includegraphics[]{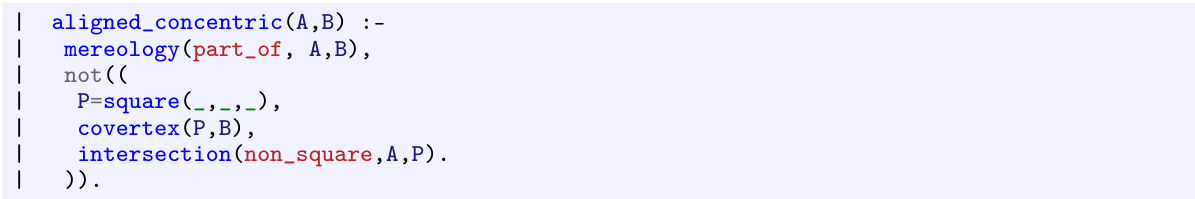}


We use CLP(QS) to prove that the definition is sufficient, by contradiction; two squares are covertex if they are aligned and share a vertex, and the relation \emph{concentric} is the geometric definition of concentricity in Table \ref{tab:encodings} that is used to evaluate the mereological definition of concentricity:

\includegraphics[]{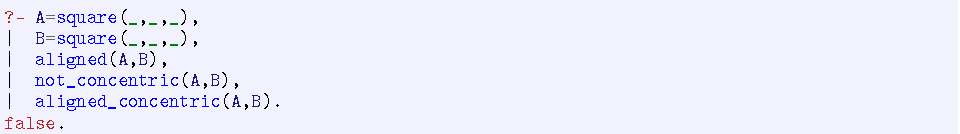}


\begin{figure}[t]
\centering
\subfigure[$A$ and $B$ are non-concentric]{\includegraphics[width=0.18\columnwidth]{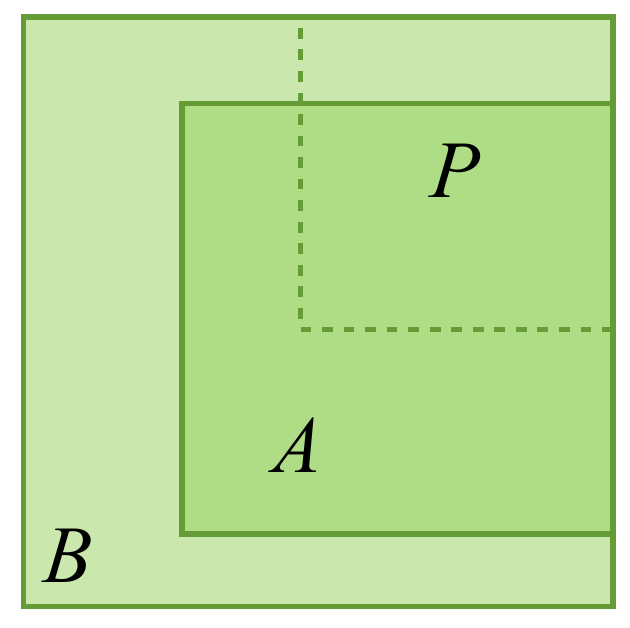}}\quad
\subfigure[$A$ and $B$ are concentric]{\includegraphics[width=0.18\columnwidth]{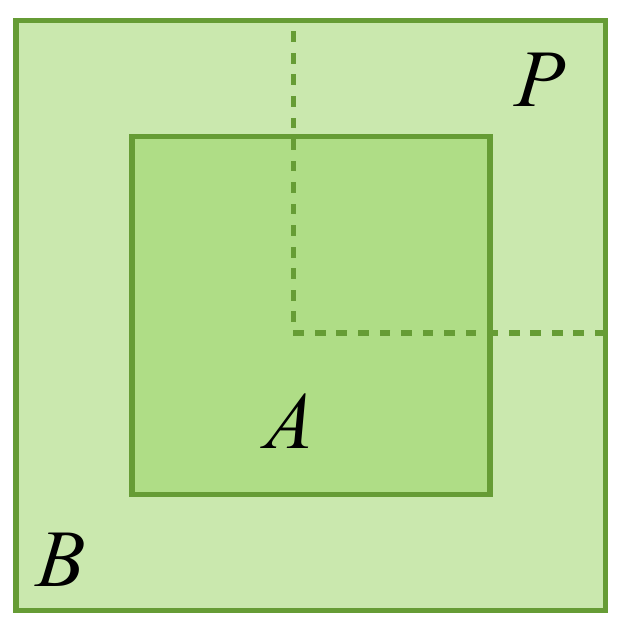}}\quad
\subfigure[$A$ is boundary concentric with $B$.]{\includegraphics[width=0.18\columnwidth]{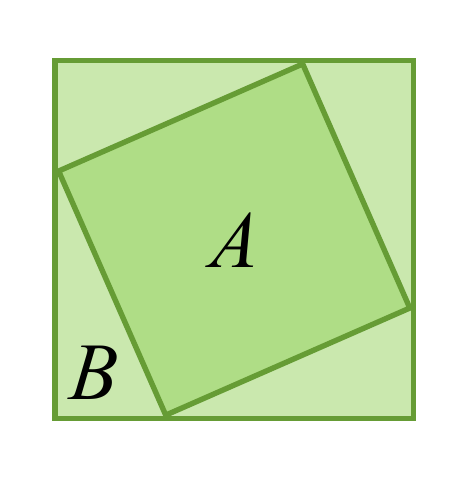}}\quad
\subfigure[$A$ and $B$ are rotated concentric.]{\includegraphics[width=0.18\columnwidth]{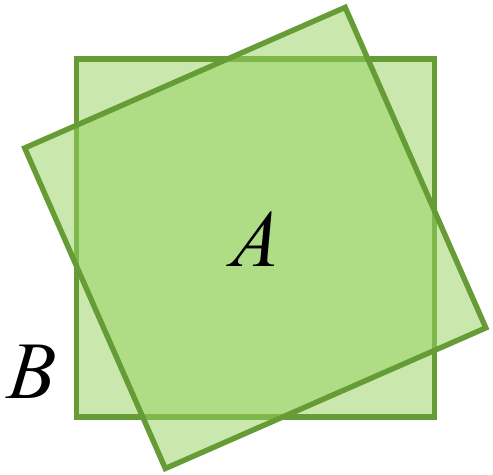}}\quad

\caption{Characterising aligned (a),(b) and rotated (c),(d) concentric squares using mereology (reproduced from [Borgo, 2013])}
\label{fig:aligned-concentric}
\end{figure}

\emph{Boundary Concentric.} Square $A$ is boundary concentric with square $B$ if: $A$ is \emph{proper part of} $B$ and there does not exist a square $Z$ such that (a) $Z$ is \emph{proper part of} $B$ (b) $A$ is \emph{part of} $Z$, and (c) $Z$ is not \emph{part of} $A$ (Figure \ref{fig:aligned-concentric}).

\includegraphics[]{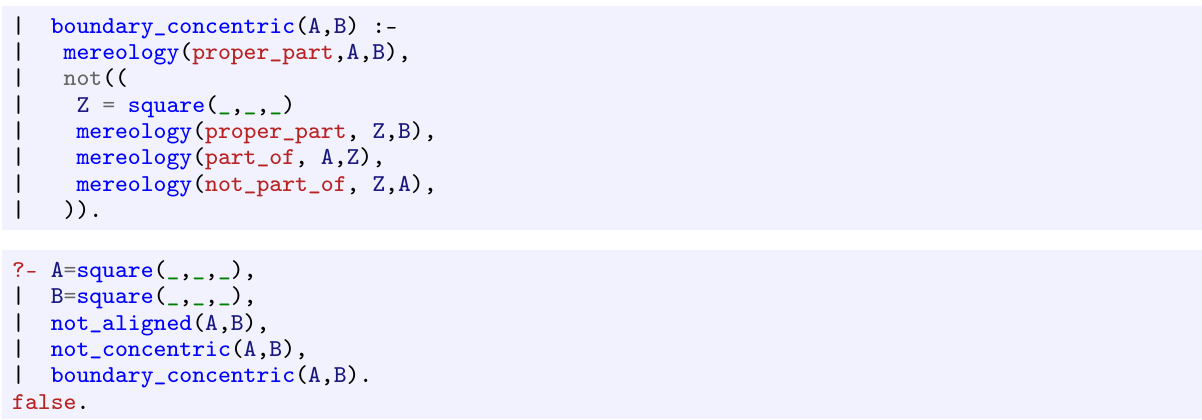}

%
%

\emph{Mereologically Concentric.} Squares $A,B$ are mereologically concentric if: $A,B$ are \emph{aligned concentric} or there exists $Q$ such that (a) $Q$ is \emph{boundary concentric} with $B$ and $Q$ is \emph{aligned concentric} with $A$ or (b) $Q$ is \emph{boundary concentric} with $A$ and $Q$ is \emph{aligned concentric} with $B$.

Having proved the mereological definitions of \emph{aligned} and \emph{boundary} concentricity, we can replace these with more efficient geometric definitions from Table \ref{tab:encodings} when proving mereological concentricity.

\includegraphics[width=\columnwidth]{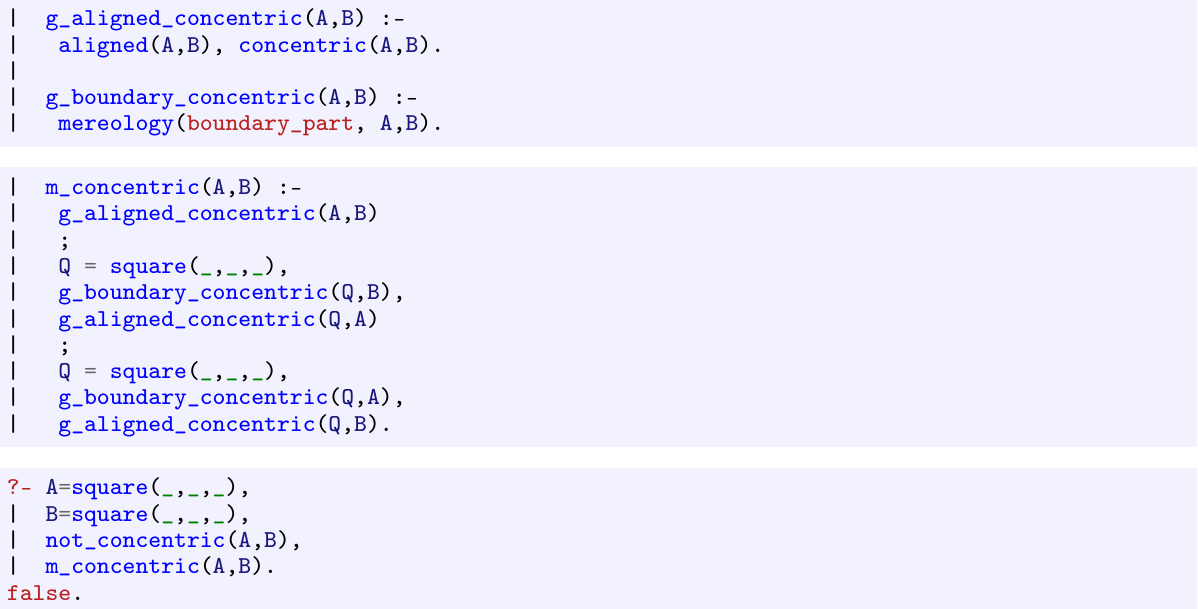}

%
%
%
%


\subsection{Didactics: Ruler--Compass and Contact Problems} 
\emph{Angle Bisector.} Let $l_a,l_b$ be line segments that share an endpoint at $p$. Draw circle $c$ at $p$. Circle $c$ intersects $l_a$ at $p_a$ and $l_b$ at $p_b$. Draw circles $c_a$ at $p_a$ and circle $c_b$ at $p_b$ such that $p$ is coincident with both $c_a, c_b$. Circles $c_a$ and $c_b$ intersect at $p$ and $p_c$. The line segment from $p$ to $p_c$ bisects the angle between $l_a$ and $l_b$ (Figure \ref{fig:angle-bisector}).

We use CLP(QS) to prove that the definition is sufficient. The relation \emph{bisects} is used for evaluation by checking if the midpoint of ($p_a$,$p_b$) is collinear with $l_c$ (i.e. idealised rulers cannot directly measure the midpoint of a line). An interactive diagram is then automatically generated that encodes the specified program (using the FreeCAD system); see Figure \ref{fig:angle-bisector-diagram}.

\includegraphics[width=\columnwidth]{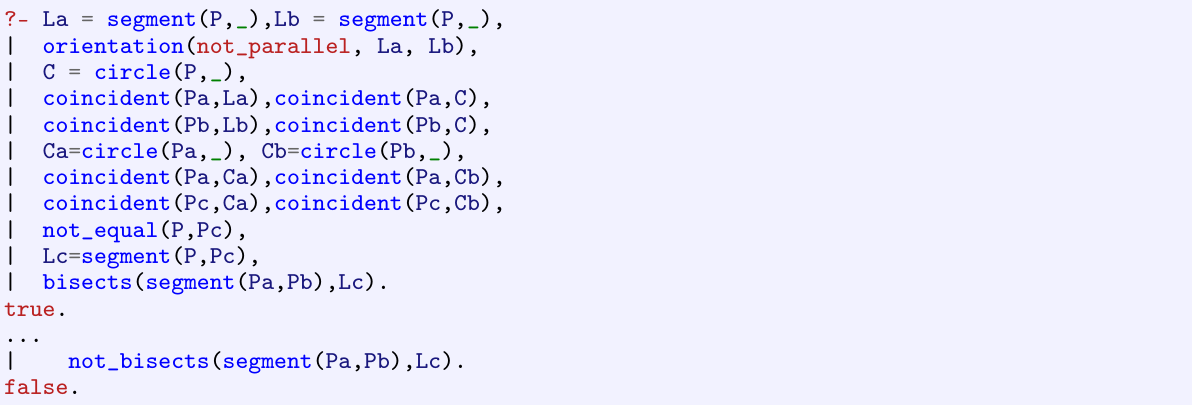}


\begin{figure}[t]
\centering
\subfigure[]{\includegraphics[width=0.32\columnwidth]{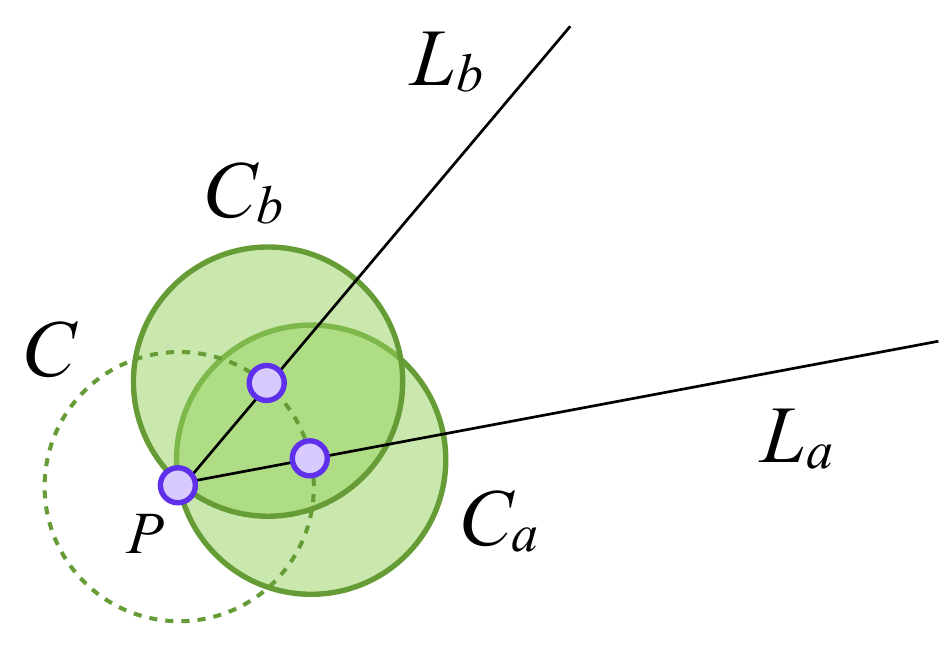}}\quad
\subfigure[]{\includegraphics[width=0.32\columnwidth]{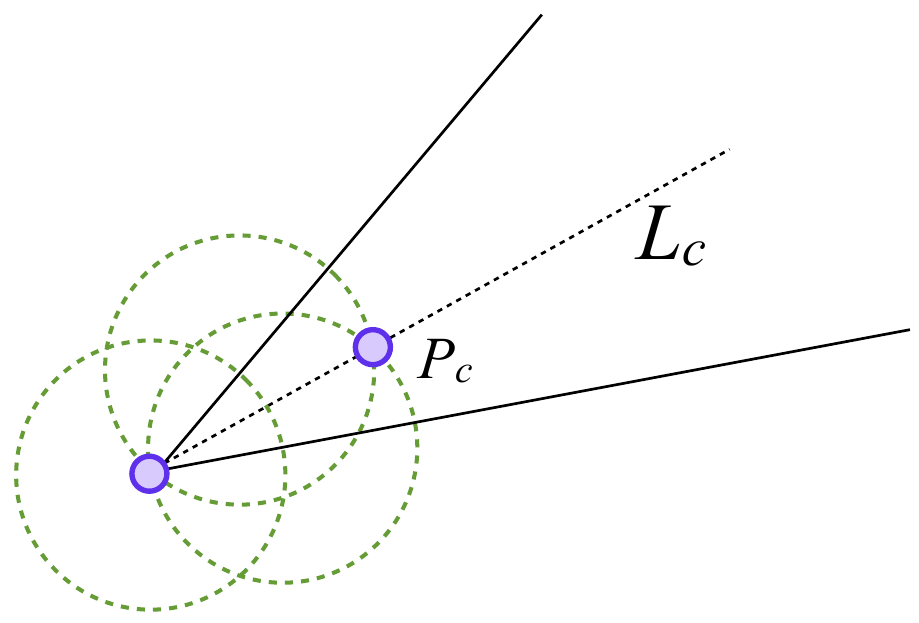}}
\caption{Ruler and compass method for angle bisection. Line $L_c$ bisects the angle between lines $L_a, L_b$.}
\label{fig:angle-bisector}
\end{figure}

\begin{figure}[t]
\centering
\subfigure[]{\includegraphics[width=0.18\columnwidth]{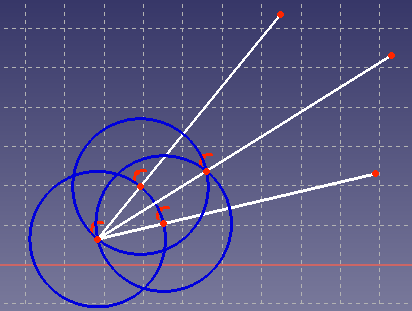}}\quad
\subfigure[]{\includegraphics[width=0.18\columnwidth]{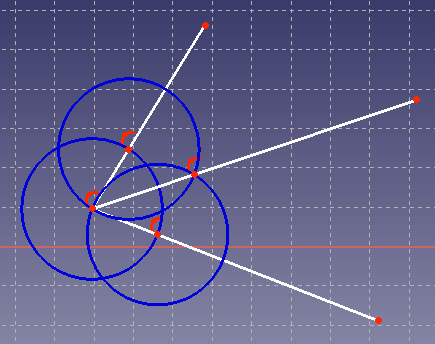}}\quad
\subfigure[]{\includegraphics[width=0.18\columnwidth]{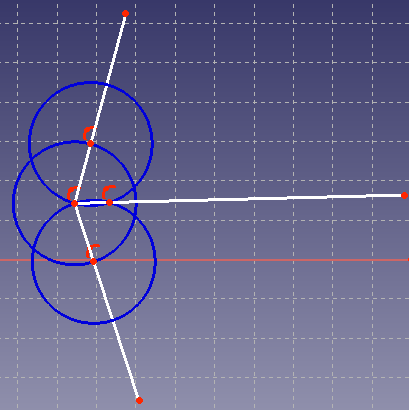}}
\caption{Interactive diagram encoding the student's constructive proof of Euclid's angle bisector theorem; as the student manipulates figures in the diagram, the other geometries are automatically updated to maintain the specified qualitative constraints.}
\label{fig:angle-bisector-diagram}
\end{figure}

\emph{Sphere contact.} Determine the maximum number of same-sized mutually touching spheres (Figure \ref{fig:spheres-eg} - note that no numeric information about the spheres is given in this benchmark problem).

\medskip

\includegraphics[width=\columnwidth]{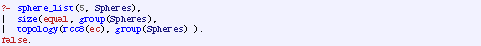}


%

\section{Conclusions}

Affine transformations provide an effective and interesting class of symmetries that can be used for pruning across a range of qualitative spatial relations. To summarise, we formalise the following knowledge as modular commonsense rules in CLP(QS): point-coincidence, line parallelism, topological and mereological relations are preserved with all affine transformations. Relative orientation changes with reflection, and qualitative distances and perpendicularity change with non-uniform scaling. Spheres, circles, and rectangles are not preserved with non-uniform scaling, with the exception of axis-aligned bounding boxes. Our algorithm is simple to implement, and is easily extended to handle more pruning cases. 

Theoretical and empirical results show that our method of pruning yields an improvement in performance by orders of magnitude over standard polynomial encodings without loss of soundness, thus increasing the horizon of spatial problems solvable with \emph{any} polynomial constraint solver. Furthermore, the declaratively formalised knowledge about pruning strategies is available to be utilised in a modular manner within other knowledge representation and reasoning frameworks that rely on specialised SMT solvers etc, e.g., in the manner demonstrated in ASPMT(QS) \citep{aspmtqs-lpnmr-2015}, which is a specialised non-monotonic spatial reasoning system built on top of answer set programming modulo theories.




%


\renewcommand{\bibname}{ References} 
\renewcommand\refname{ References} 

{\protect


\bibliographystyle{abbrvnat}
\bibliography{Clpqs-COSIT-2015-v4}

}



\end{document}